\documentclass[journal]{IEEEtran}
\usepackage{amsmath}
\usepackage{amssymb}
\usepackage{graphicx}
\usepackage{amsthm}

\newtheorem{assumption}{Assumption}

\newtheorem{theorem}{Theorem}
\usepackage{color}
\usepackage{booktabs}
\usepackage[table,xcdraw]{xcolor}
\usepackage[numbers,sort&compress]{natbib}
\usepackage{algorithm}
\usepackage{algorithmic}
\usepackage{lipsum}
\usepackage{graphicx}
\usepackage{epstopdf} 
\usepackage{multirow}
\usepackage{ulem}
\usepackage[colorlinks,linkcolor=red]{hyperref}
\usepackage{url}
\usepackage{hyperref}

\hypersetup{hypertex=true,
	colorlinks=true,
	linkcolor=blue,
	anchorcolor=blue,
	citecolor=blue}

\usepackage[caption=false, font=footnotesize]{subfig}

\makeatletter
\newcommand{\removelatexerror}{\let\@latex@error\@gobble}
\makeatother

\ifCLASSINFOpdf
\else
\fi
\begin{document}
%
\title{Incorporating Hidden Layer representation into Adversarial Attacks and Defences}
%
%
%

\author{Haojing~Shen,
        xxx,
        xxx,~\IEEEmembership{~Member,~IEEE},
        xxx,~\IEEEmembership{~Fellow,~IEEE} \thanks{Haojing Shen, Sihong Chen, and Xizhao Wang are with Big Data Institute, College of Computer Science and Software Engineering, Guangdong Key Lab. of Intelligent Information Processing, Shenzhen University, Shenzhen 518060, Guangdong, China (Email: winddyakoky@gmail.com; 2651713361@qq.com; xizhaowang@ieee.org ).} \thanks{Ran~Wang is with the College of Mathematics and Statistics, Shenzhen University, Shenzhen 518060, China and also with the Shenzhen Key Laboratory of Advanced Machine Learning and Applications, Shenzhen University, Shenzhen 518060, China. (e-mail: wangran@szu.edu.cn).} \thanks{Corresponding author: Xizhao~Wang.}}

\maketitle

\begin{abstract}
	In this paper, we propose a defence strategy to improves adversarial robustness incorporating hidden layer representation. The key of this defence strategy aims to compress or filter input's information including adversarial perturbation. And this defence strategy can be regarded as an activation function which can be applied any kind of neural networks. We also prove theoretically the effectiveness of this defense strategy under certain conditions. Besides, incorporating hidden layer representation we propose three types of adversarial attacks to generate three types of adversarial examples, respectively. The experiments show that our defence method can significantly improve the adversarial robustness of deep neural networks which achieves the state-of-the-art performance even though we do not adopt adversarial training. 

\end{abstract}

%
\IEEEpeerreviewmaketitle

\section{Introduction}

Recent deep neural networks (DNNs) make breakthroughs in many areas such as computer vision, speech recognition and so on. The power of DNNs brings people infinite reverie. Many significant and interesting works on DNNs bursting out. There are many kinds of DNNs as such convolution neural network (e.g. LeNet \cite{45_lecun1998gradient}, AlexNet \cite{33_krizhevsky2012imagenet}, ResNet \cite{42_he2016deep}), recurrent neural network (e.g. LSTM \cite{46_olah2015understanding}). As the rapid development of DNNs, more and more people focus on the security of DNNs. In particular, creating adversarial examples and defending adversarial attack are crucial techniques in the security of DNNs.

The vulnerability \cite{2_szegedy2013intriguing} of DNNs receives great attention since it has been found. Many works are focusing on this topic in literature. These researches can be roughly classified into two categories, such as adversarial attacks \cite{2_szegedy2013intriguing,6_goodfellow2014explaining,7_kurakin2016adversarial,8_papernot2016limitations,9_moosavi2016deepfool,10_carlini2017towards,11_chen2017zoo,12_moosavi2017universal,13_su2019one,14_zhao2017generating,15_tabacof2016adversarial} and adversarial defences \cite{2_szegedy2013intriguing,24_madry2017towards,43_tramer2017ensemble,44_papernot2016distillation,50_huang2015learning,51_cheng2020cat,52_song2018improving,53_zhang2019theoretically}. The methods of adversarial attack always want to create the adversarial examples which can fool the DNNs. Furthermore, the adversarial examples are usually imperceptible to a human. Inversely, adversarial defences make DNNs more robustness to adversarial examples. The adversarial attacks and adversarial defences just like a zero-sum game.

Many works follow such definition of adversarial example that adversarial example is very close to original natural images but fool the DNNs \cite{2_szegedy2013intriguing}. In the context of this definition, an adversarial example is imperceptible to human (We call it imperceptible adversarial example). However, there are some other forms of adversarial examples. For example, some researches add some patches to the original image for fooling the DNNs \cite{16_brown2017adversarial,17_liu2018dpatch,18_karmon2018lavan,19_lee2019physical}. These examples are called adversarial patches which can fool the DNNs though they are perceptive to human. Moreover, the adversarial noise images manipulated willfully by humans can fool the DNNs \cite{4_nguyen2015deep}. Though these noise images are no meaning for human, DNNs still classify them as a certain class with high confidence (we call them as unrecognizable adversarial examples). It is worth-noted that the above works mainly focus on image perturbations designed to produce mistake class label. Therefore, Sabour et al. \cite{55_DBLP:journals/corr/SabourCFF15} concentrate on the hidden layers of DNN representations and propose a method to construct an adversarial example which is perceptually similar to one image, but its hidden layer representation appears remarkably similar to a different image, one from a different class. Their works raise questions about DNN representation, as well as the explanation of adversarial examples. It is worth further study how adversarial disturbance is represented in the hidden layer and whether the information in the hidden layer can be used to construct stronger adversarial defense strategy.

Based on above works, in this paper, we analyze the characteristics of different types of adversarial examples. And we find that whether the sample is an imperceptible adversarial example, a adversarial patch or unrecognizable adversarial examples, there is an uncertainty (correct or incorrect) in the prediction results for the machine. But for human beings, as long as the disturbance does not change the original objective reality, human judgment is basically not wrong.
We emphasize that it is significantly different between human vision and DNNs to recognize the images. The human recognizes an image in pixel space while the machine does that in a special space (hidden layer) which is a machine '' vision'' space. Therefore, we propose a conjecture that, without changing the objective reality, the influence of adversarial disturbance on the prediction results depends on whether the original characteristic semantics are changed in the hidden layer propagation process.

Based on this assumption, we propose a defence mechanism, which corporating hidden layer representation to improve the adversarial robustness of the model. Our defence strategy will compress the information of input and filter many of that, including the adversarial perturbation. In fact, the proposed defense method can be regarded as an activation function, which acts in the hidden layer of the neural network and plays a role in improving the adversarial robustness of the model. And we theoretically certify the effectiveness of the defence under a simple case. Besides, we provide a certain defensive function and empirically verify the effectiveness of our defence mechanism. It is noteworthy that our defensive model achieves state-of-the-art performance. Finally, we discuss how to construct imperceptible adversarial examples, adversarial patches and unrecognizable adversarial examples by using the representation of hidden layers.

The contributions are listed as below:

\begin{enumerate}
	\item We discuss the characteristics of different types of adversarial examples (Imperceptible adversarial example, Adversarial patch and Unrecognizable adversarial example) and propose a conjecture that the representation of adversarial perturbation in hidden layer is much different with that of clean example. 
	\item Based on this conjecture, we propose an effective and efficient defence method to improve the adversarial robustness of the model. The proposed defence strategy will compress the information of input and filter many of that, including the adversarial perturbation. This defense strategy is also regarded as an activation function, and we prove the effectiveness of this defense strategy theoretically.
	\item Incorporating the representation in hidden layer, we provide three types of adversarial attack to generate imperceptible adversarial examples, unrecognizable adversarial examples and adversarial patches, which are shown in Fig. \ref{overview_123}.
	\item Our defence method, which achieves state-of-the-art performance, has $97.09\%$ accuracy under CW on MNIST, $89.03\%$ accuracy under CW on CIFAR10 and $64.38\%$ under CW on ImageNet. 
\end{enumerate}

In the rest of this paper, Section \ref{section2} introduce the related work in literature. In Section \ref{section3}, we disscuss three types of adversarial examples and propose a conjecture. In Section \ref{section4}, we propose a defence strategy and a certain defence function (activation function). In Section \ref{section5}, incorporating the representation in hidden layer, we introduce three advesarial attacks to generate three types of adversarial examples, respectively. In Section \ref{experiment}, we conduct large experiments and confirm that our methods are effective. Finally, Section \ref{conclusion} concludes the paper.



\section{Related works} \label{section2}
%

\subsection{Threat model}


Recently DNNs are powerful function and have made the advanced achievements in many domains such as Image Classify \cite{33_krizhevsky2012imagenet,34_ren2015faster}, NLP \cite{31_sutskever2014sequence,32_xu2016text}, Object Detection \cite{37_xie2017adversarial,16_brown2017adversarial,22_thys2019fooling}, Semantic classification \cite{36_hendrik2017universal,37_xie2017adversarial,38_fischer2017adversarial} and many more. With the burst spread of DNNs, more and more people pay attention to the security of DNNs. We may want to (or not want to) apply the DNNs depends on the extent to which the adversary generates the adversarial examples.


Natural language processing (NLP) make great progress in recent years. The model based on DNNs achieves state-of-the-art performance in various fields such as language modelling \cite{27_mikolov2011extensions,28_jozefowicz2016exploring}, syntactic parsing \cite{29_kiperwasser2016simple}, machine translation \cite{30_bahdanau2014neural,31_sutskever2014sequence} and many more. Then, the attention has transformed the risk of the DNNs in NLP. Recent work has shown \cite{35_papernot2016crafting} it is possible to generate adversarial examples by adding noises into texts, which fool the DNNs. This work firstly finds adversarial examples in texts, leading an arm begins between attack and defence in texts. 


In the context of object detection, recent works find various interesting adversarial examples \cite{37_xie2017adversarial}. For example, C. Xie et al. propose an algorithm called Dense Adversary Generation (DAG) to generate imperceptible adversarial examples. To decrease computing time, They utilize regional proposal network to produce the possible targets and then sum up the loss from these targets. Another form of adversarial examples in object detection is an adversarial patch \cite{16_brown2017adversarial} which generate a small image to stick it on the original image. Even these adversarial examples can attack the physical world (e.g. attacking Yolo \cite{22_thys2019fooling}).


In the semantic classification domain, the adversary is allowed manipulating fewer pixels than other domain since each perturbation is responsible for at least one-pixel segmentation \cite{36_hendrik2017universal,37_xie2017adversarial,38_fischer2017adversarial}. These attacks include non-targeted attacks and targeted attacks. However, they generate adversarial examples which are imperceptible to humans.


This paper focuses on image classification. The model based on DNNs makes excellent progress in image classification since Olga Russakovsky et al. created AlexNet \cite{39_russakovsky2015imagenet} which achieved champion in ILSVRC 2010. After that, many excellent DNNs (such as GoogleNet \cite{40_szegedy2015going}, VGG \cite{41_simonyan2014very}, ResNet \cite{42_he2016deep} etc.) are proposed in this domain. However, when people indulge in the feast of DNNs, it is found that the DNNs are incredibly vulnerable to adversarial examples \cite{2_szegedy2013intriguing}. The study of adversarial examples becomes significant in this domain. 


Usually, the successful rate of adversarial attack is related to how much distortion adversary can manipulate in original images. Szegedy et al. \cite{2_szegedy2013intriguing} use L2-norm to quantify the difference between adversarial images and original images. However, this metric is not necessarily applicable in other domain such as NLP.


According to the knowledge of adversary, adversarial attacks can be classified as white-box attacks and black-box attacks. In this paper, We suppose the adversary can access the detail of DNNs, including parameters and framework. With this strong assumption, we can construct aggressive adversarial examples and then also utilize them to black-box attacks since previous work \cite{43_tramer2017ensemble} find the adversarial examples have a property called transferability that perturbations crafted on an undefended model often transfer to an adversarially trained one.

\subsection{Adversarial example}


Szegedy et al. \cite{2_szegedy2013intriguing} firstly find the vulnerability of the DNNs. They show that original image added small perturbation, called adversarial example, can fool DNNs. They firstly describe searching adversarial examples as a box-constrained optimization problem: 

\begin{equation}
\begin{aligned}
\min \quad &\|r\|_2 \\
s.t. \quad &f(x+r) = l \\
&x+r\in [0,1]^m
\end{aligned}
\label{eq:bg:define}
\end{equation}


$f$ means a classifier (the DNN) mapping pixel space to a discrete category set. $x$ is the raw image and $r$ is the perturbation which is limited in $[0,1]$. Moreover, $l$ is a targeted label. Solving this formula, we can construct the adversarial example which the model classifies $x+r$ as $l$.

This intriguing property rises significant attention on the security of the DNNs. Then, a significant number of works \cite{2_szegedy2013intriguing,6_goodfellow2014explaining,7_kurakin2016adversarial,8_papernot2016limitations,9_moosavi2016deepfool,10_carlini2017towards,11_chen2017zoo,12_moosavi2017universal,13_su2019one,14_zhao2017generating,15_tabacof2016adversarial} try to generate all kinds of adversarial examples which can fool the DNNs such as Fast Gradient Sign Method \cite{6_goodfellow2014explaining}, Basic Iterative Method \cite{7_kurakin2016adversarial}, Jacobian-based Saliency Map Attack \cite{8_papernot2016limitations}, DeepFool \cite{9_moosavi2016deepfool}, CW \cite{10_carlini2017towards}, PGD \cite{24_madry2017towards}, One Pixel Attack \cite{13_su2019one} and so on. Goodfellow et al. \cite{6_goodfellow2014explaining} firstly propose based-gradient attack, which updates along the direction of the signal function of the pixel gradient to obtain the adversarial examples. It is a fast method since it updates one step. Based on FGSM, Kurakin et al. \cite{7_kurakin2016adversarial} propose a more powerful attack method which is multi-update to generate adversarial examples. Also using multiple iterations, Deepfool \cite{9_moosavi2016deepfool} utilize a linear approximation method to produce adversarial examples by searching the minimum distance from a clean example to an adversarial example. Looking for more powerful and non-gradient based attack method, Carlini et al. \cite{10_carlini2017towards} propose various objective functions and distance metrics to generate adversarial examples which can disable the Distillation Defense \cite{44_papernot2016distillation}. PGD \cite{24_madry2017towards} is the first-order adversary which is the most potent attack method among the first-order attack methods. Madry et al. \cite{24_madry2017towards} study the adversarial robustness of DNNs and utilize projected gradient descent to search for more aggressive adversarial examples. Their method can significantly improve resistance to most of the adversarial attacks. However, most of the researches search the adversarial examples follow the definition of adversarial examples \cite{2_szegedy2013intriguing} which are imperceptible to humans but can fool DNNs. Those researches focus on studying the small perturbation. 


Some particular adversarial examples have been found as the development of adversarial examples. Brown et al. \cite{16_brown2017adversarial} firstly introduce adversarial patch attacks for image classifiers. They design a small patch than the size of the original image and put the small patch into the original image. Those adversarial examples are called adversarial patches those that are applied in object detection \cite{17_liu2018dpatch,18_karmon2018lavan,20_sharif2016accessorize,21_eykholt2018robust}. Moreover, many works find that adversarial patches exist in the physical world \cite{3_liu2019perceptual,19_lee2019physical,22_thys2019fooling,23_chen2018shapeshifter}. The adversarial patch is not imperceptible to humans. And the perturbation usually does not confuse humans judge. However, it can fool the DNNs.


Anh et al. \cite{4_nguyen2015deep} find that some unrecognizable for humans can be classified as a class with high confidence by DNNs. Based evolutionary algorithms algorithm, they proposed a new algorithm called the multi-dimensional archive of phenotypic elites MAP-Elites \cite{25_cully2014robots}, enable them to evolve the population better. Sara et al. \cite{1_sabour2015adversarial} show that the representation of the raw image in a DNN can be operated to approximate those of other natural images by adding a minor, imperceptible perturbation to the original image. They focus on the internal layers of DNN representations. They create an adversarial example which approximated the original image, but its internal representation appears remarkably different from the original image. These researches are of great significance and raise some questions about DNN representations, as well as the vulnerability of DNN. 

\subsection{Adversarial Defense}

Adversarial defence \cite{2_szegedy2013intriguing,24_madry2017towards,43_tramer2017ensemble,44_papernot2016distillation,50_huang2015learning,51_cheng2020cat,52_song2018improving,53_zhang2019theoretically} makes DNNs more robust to adversarial examples. Papernot et al. \cite{44_papernot2016distillation} propose Defensive Distillation to defence DNN against adversarial examples. By controlling the temperature in the distillation network, they train a DNN with hard-label and train other DNN with soft-label. And the two networks have the same framework. However, this defence will fail under CW \cite{10_carlini2017towards} attack method. Adversarial training is one of the most effective defence methods. Szegedy et al. \cite{2_szegedy2013intriguing} and Madry et al. \cite{24_madry2017towards} use their attack method to generate adversarial examples. Then combining clean examples and adversarial examples, they utilize them to train a DNN, which can significant improve adversarial robustness of the DNN. However, such defence method will fail under more powerful attack methods and is time-consuming. Tramer et al. \cite{43_tramer2017ensemble} generate adversarial examples by various adversarial attacks and propose adversarial ensemble training which augments training data with perturbations transferred from other models. However, recent work \cite{54_athalye2018obfuscated} points out that adversarial training would cause obfuscated gradients that lead to a false sense of security in defences against adversarial examples. They propose three types of obfuscated gradients and design adversarial attacks that successfully attack all defensive papers in ICLR 2018 except that for CW \cite{10_carlini2017towards}.

%
%
%



\section{Adversarial Example Generation} \label{section3}
\begin{figure}[ht]
	\centering
	\subfloat[]{
		\includegraphics[width=0.8\linewidth]{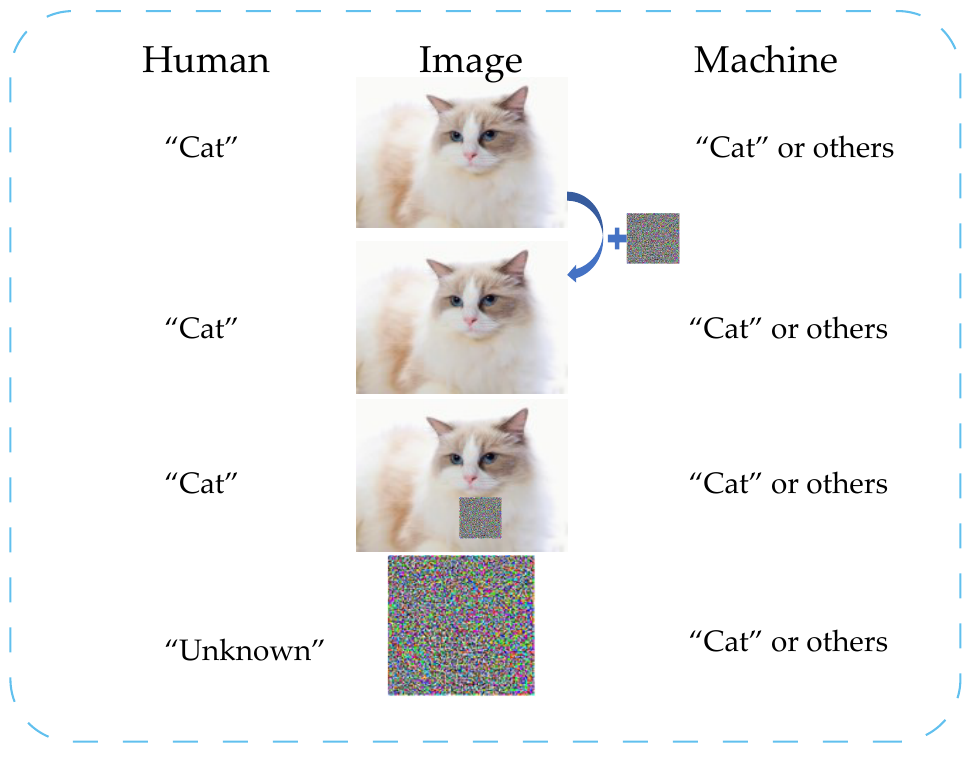}
	\label{img:human_machine_all_1}}
	
	\subfloat[]{
		\includegraphics[width=0.8\linewidth]{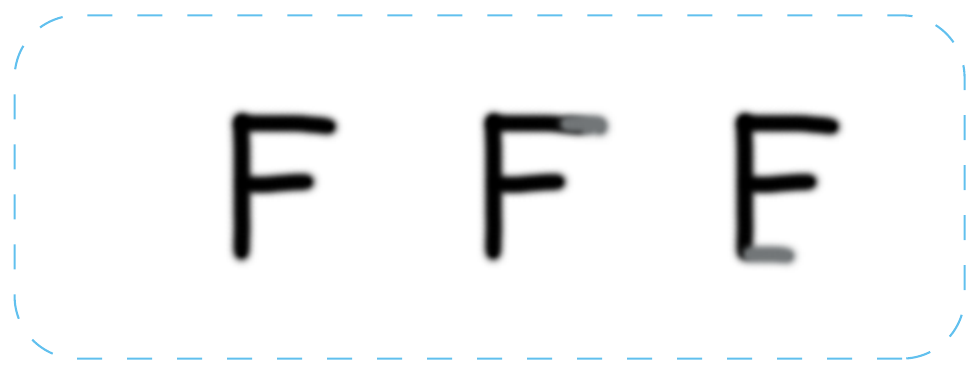}
	\label{img:human_machine_all_2}}
	
	\caption{(a) The difference between human and machine in recognition. (b) An example of perturbation changing the objective reality (For better understanding, we exaggerate the effect of the perturbation). Left: That is a English Word called "F". Middle: An image is generated by combining left picture and small perturbation. The perturbation does not change the objective reality of left picture. Right: An image is generated by combining left picture and small perturbation. However, the perturbation does change the objective reality of the original picture.}
	\label{img:human_machine_all} 
\end{figure}

In this section, to begin, we introduce three existed types of adversarial attack. These three types of adversarial attack include most of the current adversarial attack methods. By analyzing the characters of these attacks, we propose an assumption. Based this assumption, we try to explain the existent of adversarial examples in DNNs and propose three adversarial attacks and defences.

%

Many works follow the definition of adversarial examples drawn by Szedegy et al. \cite{2_szegedy2013intriguing}. We refer to this type of adversarial examples as imperceptible adversarial examples which is very close to the original image but can fool the DNN. Moreover, Nguyen et al. \cite{4_nguyen2015deep} find a new type of adversarial examples which are unrecognizable to human. We refer to it as unrecognizable adversarial examples that are hard to understand for human but is classified by a DNN with high confidence. Another type of adversarial examples is adversarial patches which are created by adding some small patches to original images \cite{16_brown2017adversarial}. These adversarial examples have a familiar character that can fool the DNN, but their original images are classified correctly by DNNs. The difference between these adversarial examples is the magnitude of change in input. The details are shown below:

\begin{itemize}
	\item Imperceptible adversarial examples: allows the adversary to change the whole image limiting in bounded-norm.
	\item Unrecognizable adversarial examples: allows the adversary to change the whole image without limiting, but the adversarial examples after adding perturbation should be unrecognizable to human.
	\item Adversarial patches: allows the adversary to change the image in a confined region.
\end{itemize}

First of all, let's think about the problem: If an example of A class become an example of B class after adding perturbation and the DNNs misclassify it, is it an adversarial example? we argue that the adversary cannot change the objective reality since all machine tasks are in the service of humanity, and the humans understand the world according to the objective reality. In practice, an adversarial example usually consists of a clean sample with a slight perturbation, e.g., imperceptible adversarial sample. It is this imperceptibility that makes it difficult for the perturbation to change the objective reality of the original example. However, the adversarial patch does not take into account this imperceptibility, which may cause the sample to become another type of sample after adding disturbance. For example, as shown in Fig. \ref{img:human_machine_all_2}, the left subfigure is an English word, called ''F''. We assume it is an objective reality (e.g. human and machine all called it as ''F''). Then we get medial subfigure by adding little pixel in the left subfigure. These perturbation does not bother humans recognition in pixel space, but the machine may be confused with other words in some feature space. We consider it adversarial example if the model classifies it wrongly. Moreover, the right subfigure also made by adding little pixel in the left subfigure. We find that in pixel space it becomes another word which is called ''E''. It is meaningless that we call a sample as an adversarial example in this case if the model classifies it wrongly.

Secondly, it is obvious that we can perceive the perturbation of adversarial patches and unrecognizable adversarial examples. Furthermore, for imperceptible adversarial examples, if the limitation of perturbation is relaxed enough, we also can perceive the small perturbation. We point out that it is significantly different between human vision and DNNs to recognize the images. It is a common phenomenon that the perturbation that hard to recognize by human but may be easy to recognize for a DNN. Inversely, the perturbation such as adversarial patches that are easy to recognize by human but maybe are hard to recognize for a DNN. As shown in Fig. \ref{img:human_machine_all_1}, we find that when humans identify clean or disturbing images, the answer is always unique, whereas DNN cannot give a definite answer. If the machine misidentified the disturbing image, such a sample is known as an adversarial example. Besides, if an adversarial example is imperceptible to human, the ''imperceptibility'' is defined in pixel space.

By summarizing the adversarial examples above, we draw the following conclusions about the adversarial examples:
\begin{itemize}
	\item The adversarial example cannot change the objective reality.
	\item The imperceptibility of the adversarial example is relative.
\end{itemize}
Therefore, we have the follow assumption:
\begin{assumption} \label{assumption}
	Without changing the objective reality, the influence of adversarial disturbance on the prediction results depends on whether the original characteristic semantics are changed in the hidden layer propagation process.
\end{assumption}

 With this opinion, an effective adversarial perturbation means that it can affect the prediction result through the propagation of hidden layer without changing the original label in pixel space.

\section{Defense Mechanism} \label{section4}
\begin{figure*}[ht]
	\centering
	\includegraphics[width=1.0\linewidth]{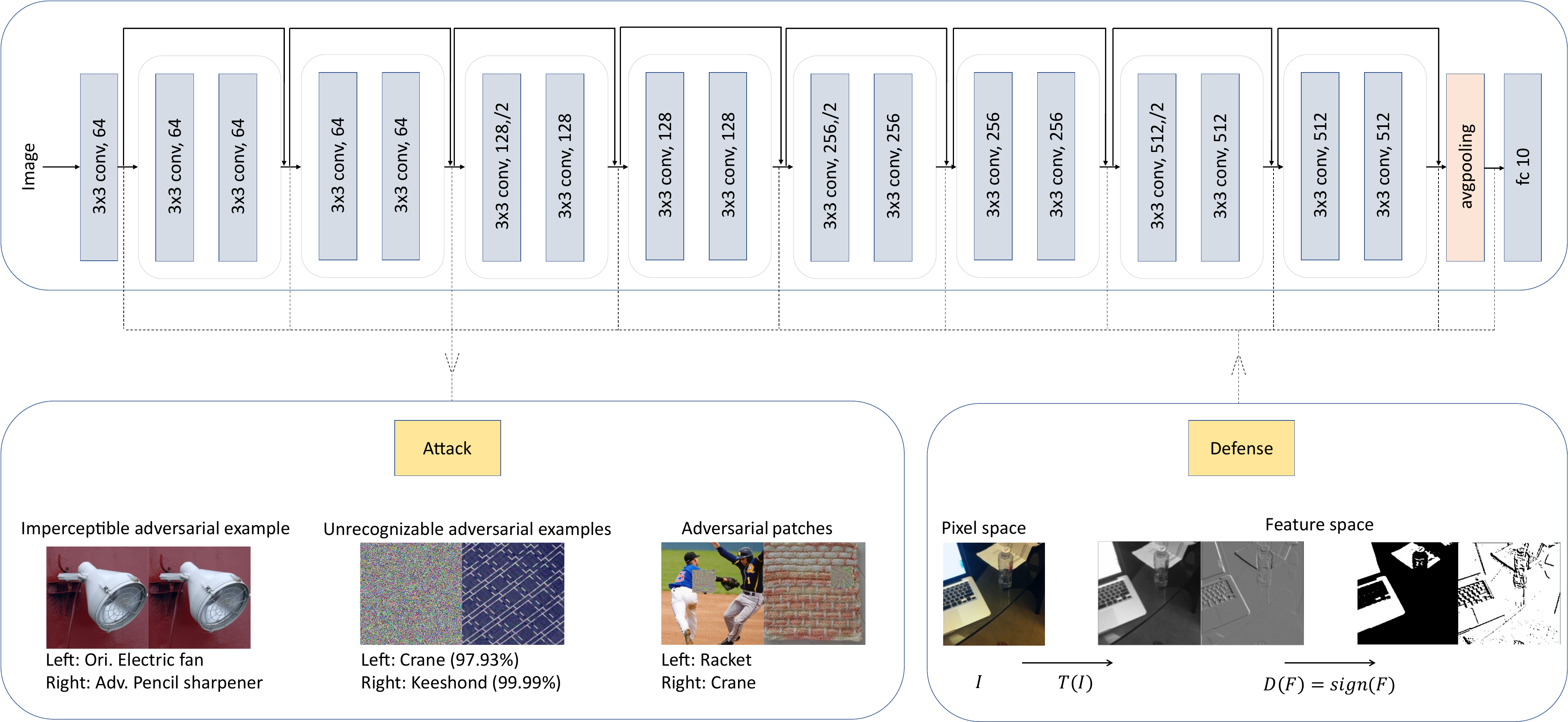}
	
	\caption{An overview of our attacks and defending strategy based the definition of generalized adversarial examples. Let Resnet18 \cite{42_he2016deep} as example. In Resnet18, we do not manipulate in residual block. We search adversarial examples (\textit{Imperceptible adversarial examples, Unrecognizable adversarial examples and adversarial patch}) by controlling the feature between residual blocks. And when defending, we add our defending strategy (e.g. \textit{Sign function}) after the first convolution which achieves the best performance. As shown in the lower right corner of the figure, the front part of the network performs feature extraction on pixel space. And the feature extraction is redundant. And our defense strategy can compress information so much that in the process we lose some information, including perturbations.}
	\label{img:overview_attack_defense} 
\end{figure*}

In this section, we introduce a defence mechanism against adversarial examples based on our assumption. In our hypothesis, the adversarial disturbance propagates through the hidden layer and finally affects the prediction results. Therefore, it is a natural idea to prevent the propagation of adversarial disturbance in the hidden layer by setting special operation in the hidden layer, so as to achieve the purpose of improving the adversarial robustness.

We use the notation from Carlini et al. \cite{10_carlini2017towards}: define $F$ to be the full neural network including the softmax function, $Z(x)=z$ to be the output of all layers except the softmax, and
\begin{equation}
 F(x) = softmax(Z(x)) = y
\end{equation}
A neural network typically consists of layers
\begin{eqnarray}
F = softmax \circ F_n \circ F_{n-1} \circ \dots \circ F_1
\end{eqnarray}
where $F_i(x)=\sigma (W_i \cdot x) + b_i$ for some activation function $\sigma$, some model weights $W_i$, and some model biases $b_i$. 

According to our assumption that adversarial perturbation can be propageted through the hidden layer, if we can prevent the disturbance from propagating forward, then the model will naturally be adversarial robust. The simple idea is to design a filter like a function and then deploy it between hidden layers. The purpose of this function is to filter and compress input information. Because DNNs usually have multiple convolution kernels, the extracted features are redundant, which provides feasible conditions for this design scheme. 

We suppose that there is a function $D$ which satisfies above characters. When $D$ compresses the input information, it would lose some of the original information, including adversarial perturbation. However, because the input information is redundant, this function $D$ does not have much impact on the performance of the model. Therefore, we can construct the defence strategy as below:
\begin{equation}
\begin{split}
FD^i(x) &= softmax(Z^i(x)) \\
&= softmax \circ F_n \circ F_{n-1} \circ \dots \circ F_{i+1} \circ D \\
&\circ F_{i} \circ \dots \circ F_1
\end{split}
\end{equation}
where $i$ means that the defence strategy ($D$) is deployed in the front of $i$-th hidden layers.

Next, we discuss two property of $D$. Firstly, $D$ must be sufficient smoothness. The training of neural network depends on the calculation of gradient. If the function $D$ is not differentiable, then the whole network is not differentiable, and then network training becomes a difficult task. Therefore, in order not to change the original training mechanism of the network, we hope that the function $D$ is uniform continuity or uniform continuity in segments. Secondly, the range of $D$ is a denumerable finite set. In traditional DNNs, the input and output are real numbers and belong to an uncountable set. The critical problem of adversarial defence is how to prevent the small perturbations slowly amplifying in a DNN. If the interval of the domain of $D$ contains the interval of the range, the perturbations will be compressed. This compression naturally leads to the loss of information. However, it is this compression effect that makes the impact of small perturbations on the results even smaller.

Let $A=F_i \circ \cdots \circ F_1$ and $B=F_n \circ \cdots \circ F_{i+1}$. The function $A$ is the front part of the DNNs. The function $B$ is the back part of the DNNs. Then, we have the following theorem:
\begin{theorem} \label{Theorem:1}
	Suppose $F_j(x) = W_j x + b_j, j=1,2,\cdots , i$, $D$ satisfies the \textit{Lipschitz continuity condition}, then $D\cdot A$ also satisfies the \textit{Lipschitz continuity condition}. Existing a constant $K$, we have
	$$
	\|D(A(x+\Delta)) - D(A(x))\|_2 \leq K\cdot \|W_i W_{i-1} \cdots W_1\|_2 \cdot \| \Delta\|_2
	$$
\end{theorem}
\begin{proof}
	Since $F_j(x) = W_j x + b_j, j=1,2,\cdots , i$, we have
	\begin{equation}
	\begin{split}
	A(x) &= F_i \cdot F_{i-1} \cdots F_1(x) \\
	&= F_i \cdot F_{i-1} \cdots F_2(W_1 x + b_1) \\
	&= F_i \cdot F_{i-1} \cdots F_3(W_2 W_1 x + W_2 b_1 + b_2) \\
	&= \cdots \\
	&= W_i W_{i-1}\cdots W_1 x + W_i W_{i-1} \cdots W_2 b_1 + \cdots W_i b_{i-1} + b_i
	\end{split}
	\end{equation},
	and 
	\begin{equation}
	\begin{split}
	A(x+\Delta) &= F_i \cdot F_{i-1} \cdots F_1(x+\Delta) \\
	&= F_i \cdot F_{i-1} \cdots F_2(W_1 x + b_1 + W_1 \Delta) \\
	&= F_i \cdot F_{i-1} \cdots F_3(W_2 W_1 x + W_2 b_1 + b_2 + W_2 W_1 \Delta) \\
	&= \cdots \\
	&= W_i W_{i-1}\cdots W_1 x + W_i W_{i-1} \cdots W_2 b_1 + \\
	&\quad \cdots W_i b_{i-1} + b_i + W_i W_{i-1} \cdots W_1 \Delta
	\end{split}
	\end{equation}
	
	$D$ satisfies the \textit{Lipschitz continuity condition}, then we have
	\begin{equation}
	\begin{split}
	\|D(A(x+\Delta)) - D(A(x))\|_2 & \leq K \cdot \| A(x+\Delta)-A(x) \|_2 \\
	&= K \cdot  \|W_i W_{i-1} \cdots W_1 \Delta\|_2 \\
	&= K \cdot \|W_i W_{i-1} \cdots W_1 \|_2 \cdot \| \Delta \|_2 
	\end{split}
	\end{equation}
	Proof done.
\end{proof}

In Theorem \ref{Theorem:1}, we give a strong supposition that $D$ satisfies the \textit{Lipschitz continuity condition}, leading $D \cdot A$ also satisfies the \textit{Lipschitz continuity condition}. With this conclusion, the effect of perturbation can be under-control within $K \cdot \|W_i W_{i-1} \cdots W_1 \|_2 \cdot \| \Delta \|_2$ where $K \cdot \|W_i W_{i-1} \cdots W_1 \|_2$ is a constant for a trained DNN. It is easy to infer that as $K \cdot \|W_i W_{i-1} \cdots W_1 \|_2 \rightarrow 0$, the adversarial perturbation would be filtered out after $D$. But for a trained model, the $\|W_i W_{i-1} \cdots W_1 \|_2$ is fixed. Therefore, we suggest that the $\|W_i W_{i-1} \cdots W_1 \|_2$ can be used to measure the adversarial robustness of the model.

According to Theorem \ref{Theorem:1}, we know that when there is no nonlinear activation function in the front part of the network, the disturbance of the input can be limited to a finite range ($K \cdot \|W_i W_{i-1} \cdots W_1 \|_2$). Therefore, we propose a more general defense model:

\begin{equation}
\begin{split}
FD^*(x) &= softmax(Z^i(x)) \\
&= softmax \circ F_n \circ F_{n-1} \circ \dots \circ F_{i+1} \circ \\
& D \circ Z_{i} \circ D \circ Z_{i-1} \circ \dots \circ D \circ Z_1(x)
\end{split}
\end{equation}
where $Z_i$ is the output of the $i$-th layer without activate function, $F_i=\sigma \circ Z_i$. In essence, we have replaced the activation function in the previous part of the network with the $D$ function. Therefore, we treat $D$ as a special activation function. 

Next, we will show how to design a reasonable function $D$. For satisfying the above two property, and we propose a specific function:
\begin{equation}
D(\boldsymbol{x}) = sign(\boldsymbol{x}) = \begin{cases}
1, &\quad \boldsymbol{x}_i \geq 0 \\
-1, &\quad \boldsymbol{x}_i < 0
\end{cases}
\label{eq:D}
\end{equation}
Since Eq. \ref{eq:D} is bounded, $D$ satisfies \textit{Lipschitz continuity condition}. And the range of $D$ is $\{1,-1\}$. In section \ref{experiment}, we will show that applying Eq. \ref{eq:D} can significantly improve the adversarial robustness of the model at the same time, keeping accuracy on clean examples.

By observing Eq. \ref{eq:D}, it is found that the equation is a piecewise differentiable function whose non-differentiable point is the origin. The defence function $D$ is too simplistic and extreme. Because the input is all integers or all negative Numbers, the model loses all available information through this function, resulting in network convergence failure. In the experiment section \ref{experiment}, we will introduce that the deeper the network layer where function $D$ (Eq. \ref{eq:D}) is located, the worse the performance of the network under clean samples. Therefore, based on equation Eq. \ref{eq:D}, we propose a more general defence function as following:
\begin{equation}
D(\boldsymbol{x}) = sign(\boldsymbol{x}) = \begin{cases}
b_u, &\quad \boldsymbol{x}_i \in [a_u, +\infty] \\
..., &\quad \boldsymbol{x}_i \in [...,...] \\
b_l, &\quad \boldsymbol{x}_i \in [a_l, -\infty]
\end{cases}
\label{eq:D-1}
\end{equation}
where $b_u$ is the upper bound, $b_l$ is the lower bound, and $\{a_u,..,a_l\}$ are the non-differentiable points. The details are shown in Algorithm \ref{alg:defense}. And the overview are shown in Fig. \ref{img:overview_attack_defense}.

\begin{algorithm}[]
	\caption{Defense Algorithm}
	\begin{algorithmic}[1]
		\REQUIRE An dataset $\mathcal{D}=\{(\boldsymbol{x}^{(i)}, y^{(i)})\}_{i=1}^N$; A network $FD^l=softmax \circ F_n \circ \cdots \circ F_{l-1} \circ D \circ F_l \cdots \circ F_1$
		\STATE Initialize hyperparameter including learing rate, epoches, optimizer, batchsize $m$.
		\FOR {i=1 to epoches}
		\FOR {$k$ steps}
		\STATE Sample $m$ samples $d=\{\boldsymbol{x}^{(1)},\boldsymbol{x}^{(2)},\cdots, \boldsymbol{x}^{(m)}\}$
		\STATE Update the parameters by descending its stochastic gradient:
		$$\bigtriangledown \frac{1}{m}\sum_{i=1}^{m} {y^{(i)} \cdot log(FD^l (\boldsymbol{x}^{(i)}))}$$
		\ENDFOR
		\ENDFOR
		
		\STATE Output $FD^l$
	\end{algorithmic}
	\label{alg:defense}
\end{algorithm}

\section{Attack Mechanism} \label{section5}

In this section, we introduce how to generate imperceptible adversarial examples, adversarial patches and unrecognizable adversarial example by incorporating hidden layer representation in a DNN. The key idea is maximizing distance between clean example and adversarial example in hidden layer.

In Section \ref{section4}, a fully neural network consists of layers
\begin{equation}
\begin{split}
F(x) = softmax \circ B \circ A(x)
\end{split}
\end{equation}
where $A=F_i \circ \cdots \circ F_1$ and $B=F_n \circ \cdots \circ F_{i-1}$. Then, adversarial example can be easily generated as follows:
\begin{equation}
\begin{aligned}
\max \quad &\| A(x+\Delta) - A(x)\|_p \\
s.t. \quad &F(x+\Delta) \not= y
\label{eq:attack_original}
\end{aligned}
\end{equation}
It is noteworthy that the adversarial example generated by Eq. \ref{eq:attack_original} is unknown, e.g, the pixel value can be any value. It is significantly different with imperceptible adversarial example, adversarial patch and unrecognizable adversarial example. However, we can construct these three types of adversarial example by adding additional restricted conditions. For example, Imperceptible Adversarial Example can be generated by the following formulation:
\begin{equation}
\begin{aligned}
\min \quad &\|\Delta\|_p - \alpha \cdot \| A(x+\Delta) - A(x)\|_p \\
s.t. \quad &F(x+\Delta) \not= y \\
&\|\Delta\|_p \leq \delta
\label{eq:attack-v1}
\end{aligned}
\end{equation}
where $\alpha$ is the hyperparameter and $\delta$ is the small value for humans perception. Eq. \ref{eq:attack-v1} try to find an adversarial example which is imperceptible to human by minimizing $\boldsymbol{\delta}$ and maximize $\| A(x+\Delta) - A(x)\|_p$. 


The second adversarial attack, which generate unrecognizable adversarial example, makes a DNN classifies a natural unrecognized image for human or a natural image from other domain with high confidence. The detailed formula is described as below:
\begin{equation}
\begin{aligned}
\min \quad &\| A(x_{in}) - A(x_{out}+\Delta)\|_p \\
s.t. \quad &F(x+\Delta) \not= y
\end{aligned}
\label{eq:attack-v2}
\end{equation}
$x_{in}$ represents a natural image which comes from the original domain. And $x_{out}$ represents an image which comes from some domain or original domain. If $x_{out}$ is unrecognizable to human, then we can create an adversarial example which is unrecognizable for human but is classified by a DNN with high confidence. In that case, the attack is similar to the work \cite{4_nguyen2015deep} that produces images that are unrecognizable for humans, but a DNN believes to be recognizable objects with high confidence. If $x_{out}$ came from original domain, then this type attack is similar to the work \cite{1_sabour2015adversarial} which shows that the representation of a raw image in a DNN can be operated to approximate those of other natural images by adding a minor, imperceptible perturbation to the original image. Referring to \cite{1_sabour2015adversarial}, $x_{in}$ is guided image, and $x_{out}$ is the source image.

The third version is adversarial patch attack described as below:
\begin{equation}
\begin{aligned}
\min \quad &-\| A(x) - A(x+P_\Delta)\|_p \\
s.t. \quad &F(x+P_\Delta) \not= y \\
&P_{\Delta} is\ a\ small\ patch
\end{aligned}
\label{eq:attack-v3}
\end{equation}

In Eq. \ref{eq:attack-v3}, $P_{\boldsymbol{\Delta}}$ is a smaller patch than the original image. This attack allows an adversary to manipulate pixels in the confined region (e.g. $P_{\boldsymbol{\Delta}}$) but is no limit to choose the values yielding the max or min value. Moreover, $P_{\boldsymbol{\Delta}}$ has a specific shape such as square, circle and so on. In practical, when optimizing the Eq. \ref{eq:attack-v3}, we need to know where $P_{\boldsymbol{\Delta}}$ locate in and what size that is. Usually, the position and size of $P_{\boldsymbol{\Delta}}$ can be randomly initialized as long as the patch does not change the objective reality of the original image. Therefore, the size of $P_{\boldsymbol{\Delta}}$ always is much smaller than that of the original image.
\section{Experiments}\label{experiment}

In this section, we use Python 3.6 to conduct the proposed defence strategy and three types of adversarial attack methods. Using toolbox Advertorch \cite{49_ding2019advertorch}, we implement several adversarial attack methods as a baseline. Firstly, we evaluate the proposed defence strategy on MNIST \cite{47_deng2012mnist}, CIFAR10 \cite{48_krizhevsky2009learning} and IMAGENET \cite{39_russakovsky2015imagenet}. Then, for imperceptible adversarial examples, we compared our method with existed attack methods. For unrecognizable adversarial examples and adversarial patches, we generate various types of them using the proposed method.

MNIST dataset consists of 60,000 training samples and 10,000 samples, each of which is a 28x28 pixel handwriting digital image. CIFAR10 dataset is composed of 60,000 32x32 colour image, 50,000 for training and 10,000 for testing. ImageNet is large colour image dataset which consists of more than 1,400,000 natural images and 1000 classes. 

For any sample from MNIST, CIFAR10 and ImageNet, we make a preprocess on it, e.g., normalization described as below:
$$
\boldsymbol{I} = \frac{\boldsymbol{I}-\boldsymbol{\mu}}{\boldsymbol{\sigma}}
$$
where $\boldsymbol{\mu}$ is the mean pixel value of the whole dataset and $\boldsymbol{\sigma}$ is the standard deviation of that. The detail is showed in Table \ref{tab:initial}.

\begin{table}[ht]
	\centering
	\renewcommand\arraystretch{1.0}
	\caption{THE MEAN VALUE AND STANDARD DEVIATION IN MNIST, CIFAR10 AND IMAGENET}
	\label{tab:initial}
	\begin{tabular}{lllll}
		\toprule
		& $\boldsymbol{\mu}$                           & $\boldsymbol{\sigma}$                        &  &  \\ \midrule
		MNIST    & {[}0.1307{]}                 & {[}0.3081{]}                 &  &  \\ \midrule
		CIFAR10  & {[}0.4914, 0.4822, 0.4465{]} & {[}0.2023, 0.1994, 0.2010{]} &  &  \\ \midrule
		ImageNet & {[}0.485, 0.456, 0.406{]}    & {[}0.229, 0.224, 0.225{]}    &  &  \\ \bottomrule
	\end{tabular}
\end{table}

\subsection{Adversarial Defense}

\paragraph{MNIST} We use the part of convolution layers as function $T$ and use the part of linear layers as classifier $f$. We set epochs as 30. The learning rate is $0.01$ at the beginning, and then half every 10 epochs. The momentum is $0.95$. The batch size is $64$. The results of the experiment are shown in Tabel \ref{tab:defense_mnist}. We find that our defence strategy almost invalidates all adversaries. At the same setting, our defending method can effectively defence gradient-based attacks (such as PGD) or non-gradient-based attacks (such as CW). 

\begin{table}[ht]
	\centering
	\renewcommand\arraystretch{1.0}
	\setlength{\tabcolsep}{4pt}
	\caption{Defense on MNIST: Performance of our defending strategy against different adversaries for $\epsilon=0.3$}
	\label{tab:defense_mnist}
	\begin{tabular}{llll||l||l}
		\toprule
		Method & Steps                                      & $\epsilon$ & Natural                                         & Adv. Training \cite{24_madry2017towards}                                & Our                                             \\ \midrule
		Natural     & - & -       & 0.9849 & \textbf{0.9923} & 0.9711 \\
		FGSM        & -    & 0.3     & 0.8239      & \textbf{0.9795}     & 0.9712                                 \\
		L2BIA       & 100      & 0.3     & 0.9758      & 0.9503       & \textbf{0.9712}                                 \\
		PGD         & 100   & 0.3     & 0.5838        & 0.9507   & \textbf{0.9547}                 \\
		CW          & 100  & 0.3     & 0.6052     & 0.7813         & \textbf{0.9709} \\ \bottomrule                                
	\end{tabular}
\end{table}

\paragraph{CIFAR10} Considering a more complicated case, we use ResNet18 as our basic neural network to train on CIFAR10. We set epochs as $30$. The learning rate is $0.01$ at the beginning, and then half every 10 epochs. The momentum is $0.95$. The batch size is $64$. We try to insert function $D$ into ResNet18. The part of ResNet18 before $D$ is considered as $T$. Moreover, the part of ResNet18 after $D$ is considered as $f$. The results of the experiment are shown in Tabel \ref{tab:defense_cifar10}. As Table \ref{tab:defense_cifar10} showed, all adversaries fail to attack our model. 

\begin{table}[ht]
	\centering
	\renewcommand\arraystretch{1.0}
	\setlength{\tabcolsep}{4pt}
	\caption{Defense on CIFAR10: Performance of our defending strategy against different adversaries for $\epsilon=0.03$}
	\label{tab:defense_cifar10}
	\begin{tabular}{llll||l||l}
		\toprule
		Method & Steps                                      & $\epsilon$ & Natural                                         & Adv. Training \cite{24_madry2017towards}                              & Our                                             \\ \midrule
		Natural    & - & -    & 0.8892 & 0.8613 & \textbf{0.892} \\
		FGSM   & -    & 0.03     & 0.743    & 0.7151       & \textbf{0.8930}            \\
		L2BIA     & 100  & 0.03     & 0.8775     & 0.8595   & \textbf{0.8766}    \\
		PGD      & 100   & 0.03     & 0.4586  & 0.6943    & \textbf{0.8927}     \\
		CW & 100 & 0.03     & 0.3595  & 0.0185  & \textbf{0.8903}       \\ \bottomrule                         
	\end{tabular}
\end{table}

\paragraph{ImageNet}
In this part, we apply our defence method in ImageNet dataset. We adopt ResNet50 as our model and set epochs as $90$. Then learning rate is $0.01$ at the beginning and then is adjusted according to this formulation $lr=0.01 * 0.1^{\lfloor epoch/30\rfloor}$. The momentum is $0.95$. The batch size is $256$. The defensive function is deployed in the back of the first convolutional layer. The detail results of the experiments are shown in Table \ref{tab:defense_ImageNet}. As a result, all adversaries fail to attack our defensive model. It is noted that the accuracy under different attacks is almost the same with that on clean samples. It demonstrates that our defence method can successfully extend more realistic and complicated case.

\begin{table}[ht]
	\centering
	\renewcommand\arraystretch{1.0}
	\setlength{\tabcolsep}{4pt}
	\caption{Defense on ImageNet: Performance of our defending strategy against different adversaries for $\epsilon=0.03$}
	\label{tab:defense_ImageNet}
	\begin{tabular}{llll||l||l}
		\toprule
		Method & Steps                                      & $\epsilon$ & Natural                                         & Adv. Training \cite{24_madry2017towards}                              & Our                                             \\ \midrule
		Natural    & - & -    & \textbf{0.7198} & 0.4922 & 0.6453 \\
		FGSM     & -     & 0.03     & 0.05    & 0.0010       & \textbf{0.6443} \\
		PGD      & 100   & 0.03     & 0.0003  & 0.0010    & \textbf{0.6372}     \\
		CW & 100 & 0.03     & 0.0003  & 0.0009  & \textbf{0.6438}       \\ \bottomrule                         
	\end{tabular}
\end{table}

\paragraph{Analysis} We find that our defending method will decrease the accuracy of the model on clean samples. And all adversaries fail to attack our model. Therefore, if we can prevent the drop in the accuracy of the model on clean samples, the adversarial robustness of the model can be further improved. To search the best position where the function $D$ is, we do a simple test in the network layers of LeNet and ResNet18. As shown in Fig. \ref{ex:defense_mnist_cifar10}, we find that our defence strategy is best deployed at the front of the network. The further function $D$ is deployed, the lower the recognition rate of the model on clean samples, and the ability of the model to resist the attack decreases accordingly. 

Moreover, the recognition rate of our defence model for clean samples is very similar to that of the model for adversarial samples. This indicates that the adversaries have failed to attack the defence model. We present that function $D$ acts to compress the information. After the feature passes through the function $D$, much information is lost, including adversarial perturbation. In the latter part of the neural network, information is extracted step by step. If the function $D$ is placed in the back half of the neural network, a large amount of useful information will be lost, making the performance of the model worse.

\begin{figure}[ht]
	\centering
	\subfloat[LeNet network with defense strategy]{
		\includegraphics[width=6.65cm]{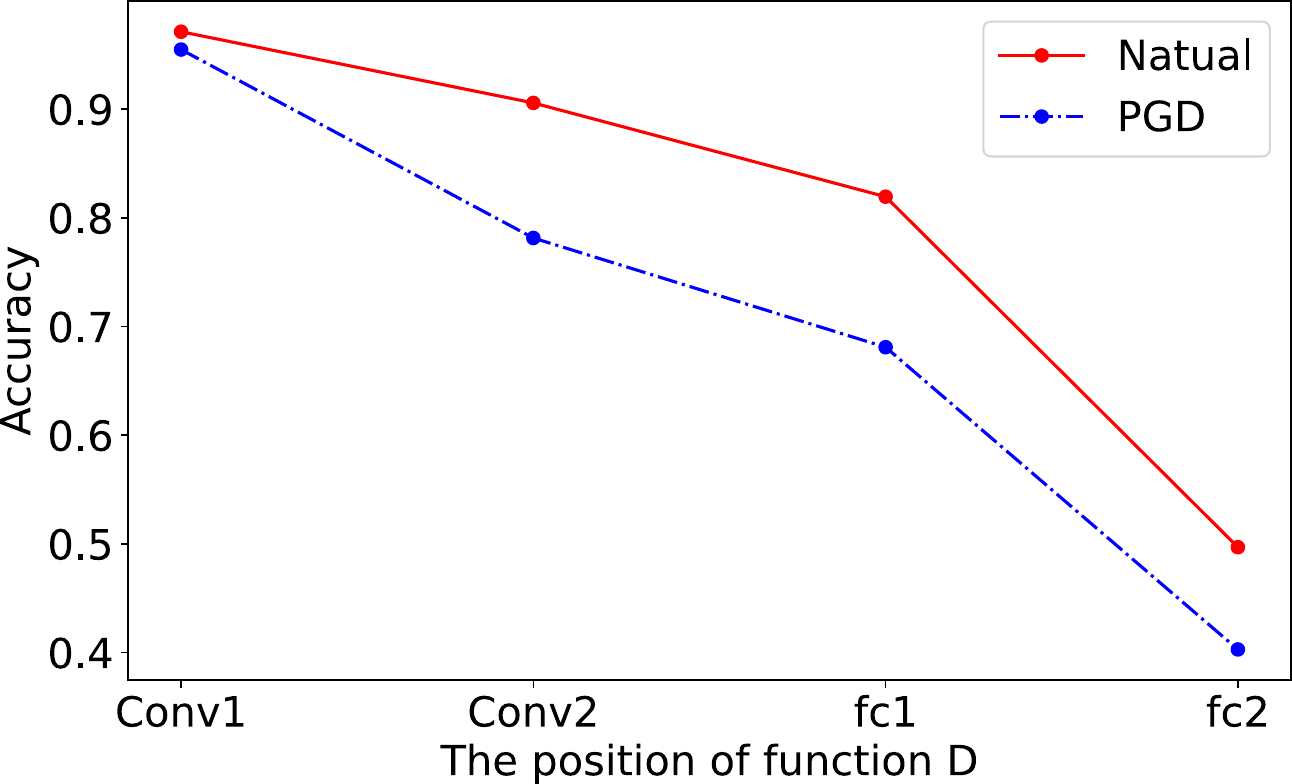}}
	
	\subfloat[Resnet18 network with defense strategy]{
		\includegraphics[width=7.0cm]{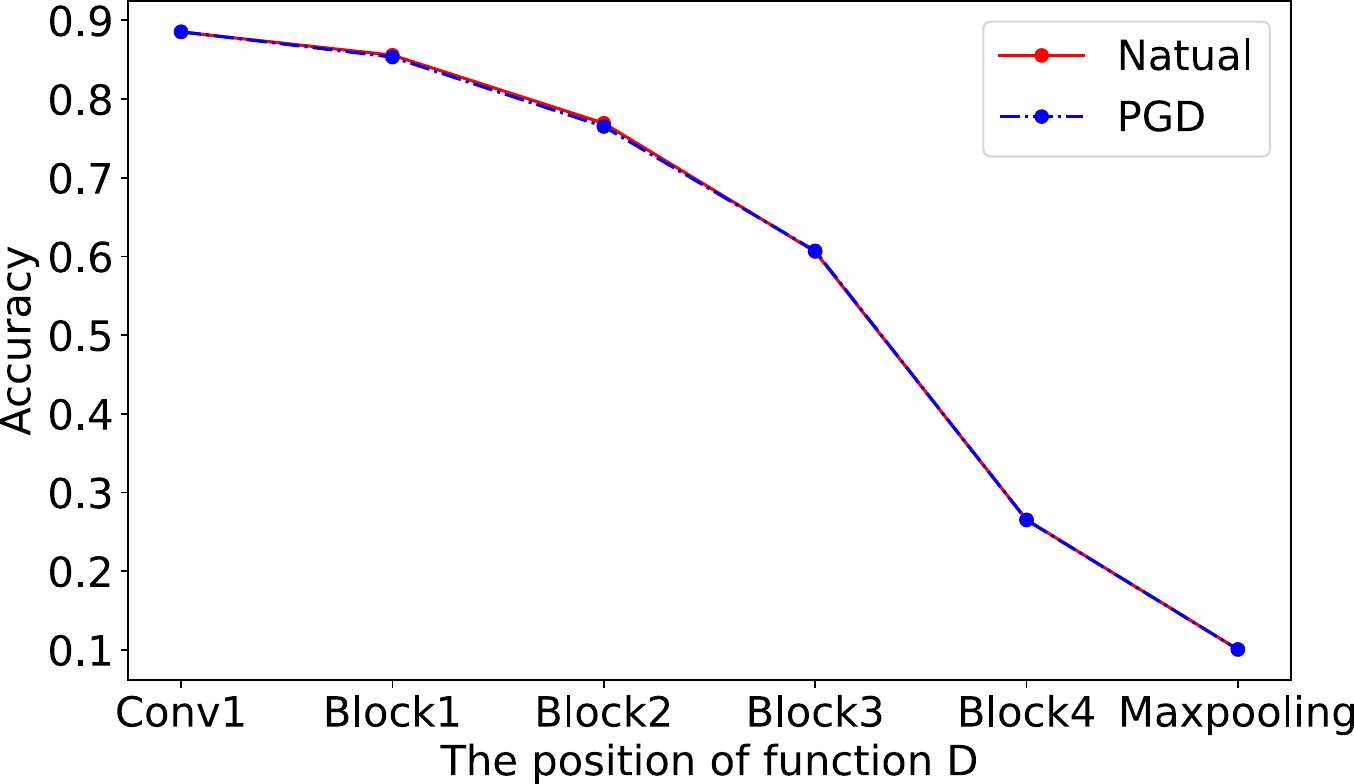}}
	
	\caption{The effect of setting location of function D on network performance. (a) Lenet have five layers include two convolutions and three connective layers. (b) ResNet18 have a convolution layer, a fully connective layer and four blocks.}
	\label{ex:defense_mnist_cifar10}
\end{figure}

Moreover, we analyze why function $D$ can effectively play a defensive role by visualizing the changes in the input found after convolution. We use the trained ResNet50 network to visualize the input, the output of the first layer of convolution, and the output after function $D$. Function $D$ is Eq. \ref{eq:D}. We compare the degree of contamination of inputs under different disturbances. As Fig. \ref{ex:visual_defense_conv} shown, when function $D$ filters the information, it filters the adversarial perturbations to some extent. When $\epsilon \leq 0.06$, function $D$ filters out adversarial perturbations almost entirely (Columns 2 and 3). In the case without function $D$, the input is heavily contaminated (column 1 and column 4).
\begin{figure}[ht]
	\centering
	\includegraphics[width=0.9\linewidth]{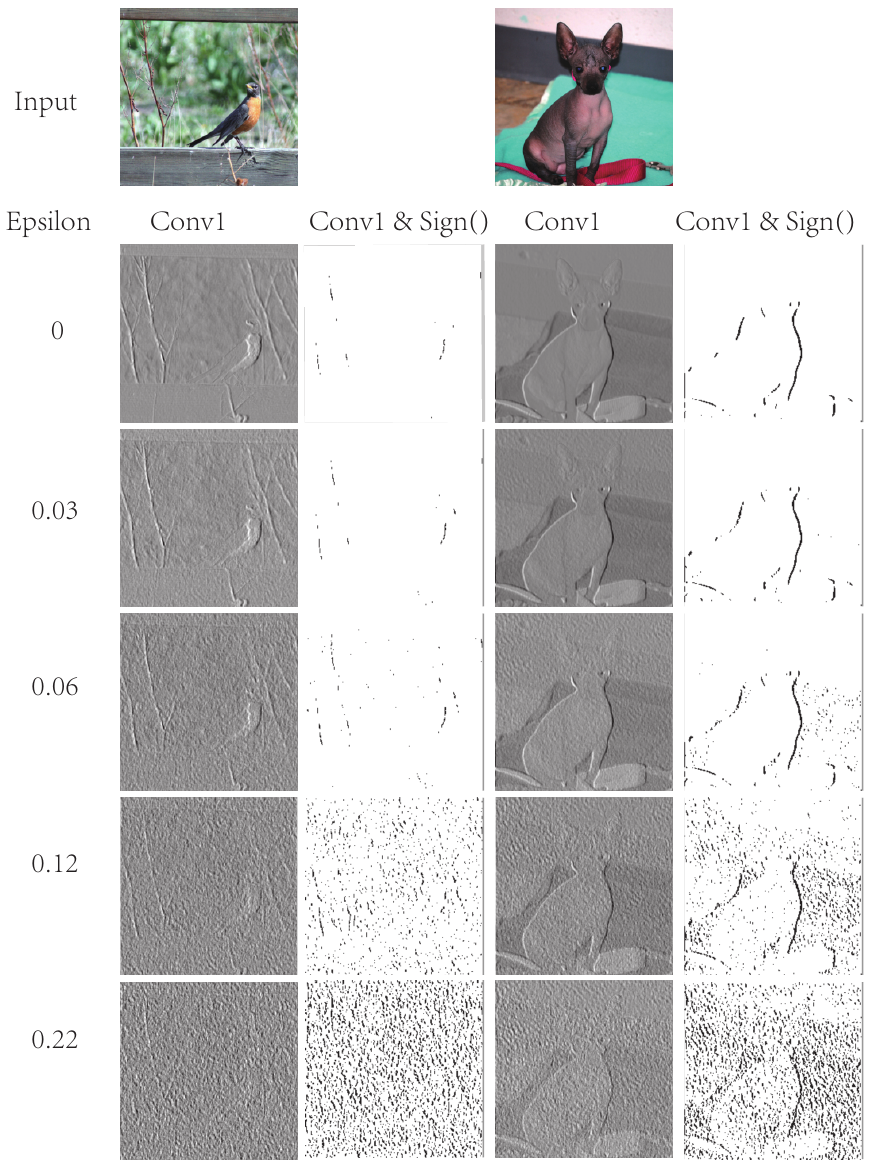}
	\caption{We use the trained ResNet50 network to visualize the input, the output of the first layer of convolution and the output after function $D$. Function $D$ is Eq. \ref{eq:D}. We compare the degree of contamination of inputs under different disturbances.}
	\label{ex:visual_defense_conv}
\end{figure}

\subsection{Imperceptible Adversarial Examples}

\begin{table}[ht]
	\centering
	\renewcommand\arraystretch{1.0}
	\setlength{\tabcolsep}{14pt}
	\caption{A comparison of imperceptible adversarial examples with various adversarial attacks on MNIST.}
	\label{tab:result_v1_mnist}
	\begin{tabular}{lllll}
		\toprule
		Method      & Steps & Epsilon & Accuracy        &  \\ \midrule
		Natural     & -     & -       & 0.9849          &  \\
		FGSM        & -     & 0.3     & 0.8239          &  \\
		PGD         & 100   & 0.3     & \textbf{0.5838} &  \\
		SinglePixel & -     & -       & 0.9556          &  \\
		CW          & 100   & 0.3     & 0.6052          &  \\
		LBFGS       & 100   & 0.3     & 0.9972          &  \\
		L2BIA       & 100   & 0.3     & 0.9758          &  \\ \midrule
		Our         & 100   & 0.3     & \textbf{0.5995} &  \\ \bottomrule
	\end{tabular}
	
\end{table}

\begin{table}[ht]
	\centering
	\renewcommand\arraystretch{1.0}
	\setlength{\tabcolsep}{7pt}
	\caption{A comparison of imperceptible adversarial examples with various adversarial attacks on CIFAR10.}
	\label{tab:result_v1_cifar10}
	\begin{tabular}{llll||ll}
		\toprule
		Method      & Steps & Epsilon & Accuracy        & Epsilon & Accuracy        \\ \midrule
		Natural     & -     & -       & 0.8892          & -       & -               \\
		FGSM        & -     & 0.3     & 0.4221          & 0.03    & 0.743           \\
		PGD         & 100   & 0.3     & \textbf{0.0049} & 0.03    & 0.4586          \\
		SinglePixel & -     & -       & 0.8814          & -       & 0.8818          \\
		CW          & 100   & 0.3     & 0.3604          & 0.03    & \textbf{0.3595} \\
		LBFGS       & 100   & 0.3     & 0.9361          & 0.03    & 0.9361          \\
		L2BIA       & 100   & 0.3     & 0.7739          & 0.03    & 0.8775          \\ \midrule
		Our         & 100   & 0.3     & \textbf{0.0303} & 0.03    & \textbf{0.4669} \\ \bottomrule
	\end{tabular}
\end{table}

\begin{table}[ht]
	\centering
	\renewcommand\arraystretch{1.0}
	\setlength{\tabcolsep}{14pt}
	\caption{A comparison of imperceptible adversarial examples with various adversarial attacks on ImageNet.}
	\label{tab:result_v1_imagenet}
	\begin{tabular}{lllll}
		\toprule
		Method      & Steps & Epsilon & Accuracy        &  \\ \midrule
		Natural     & -     & -       & 0.696           &  \\
		FGSM        & -     & 0.03    & 0.08            &  \\
		PGD         & 100   & 0.03    & \textbf{0.005}  &  \\
		SinglePixel & -     & -       & 0.6957          &  \\
		CW          & 100   & 0.03    & \textbf{0.4797} &  \\
		LBFGS       & 100   & 0.03    & 0.7929          &  \\ 
		L2BIA       & 100   & 0.03    & 0.692           &  \\ \midrule
		Our         & 100   & 0.03    & \textbf{0.4907} &  \\ \bottomrule
	\end{tabular}
\end{table}

In this subsection, we test the first attack method (e.g. Eq. \ref{eq:attack-v1}) on MNIST, CIFAR10 and ImageNet. Our tiny attack method is a novel adversarial attack method. We compare it with existed methods such as FGSM \cite{6_goodfellow2014explaining}, PGD \cite{24_madry2017towards}, SinglePixel \cite{13_su2019one}, CW \cite{10_carlini2017towards}, LBFGS \cite{2_szegedy2013intriguing} and L2BIA \cite{7_kurakin2016adversarial}. We utilize the LeNet \cite{45_lecun1998gradient}, ResNet18 \cite{42_he2016deep} and pre-trained VGG19 \cite{41_simonyan2014very} to test MNIST, CIFAR10 and ImageNet, respectively. We set $\epsilon$ as 0.3 when attacking on MNIST and 0.03 on ImageNet. Moreover, we set respectively $\alpha$ as 0.3 and 0.03 when attacking CIFAR10. For multi-step attacks, let the number of iterations uniformly to 100 steps.

In the best case, our attack method can drop the accuracy down to $59.95\%$ (Table \ref{tab:result_v1_mnist}) with $\epsilon$=0.3 on MNIST. Our attack method is second only PGD. For more challenging cases, such as imperceptible adversarial examples on CIFAR10, our method drop the accuracy down to $3.03\%$, with $\epsilon$=0.3 (Table \ref{tab:result_v1_cifar10}) which is second only PGD, and $46.69\%$ with $\epsilon$=0.03 (Table \ref{tab:result_v1_cifar10}) which is second only PGD and CW. Moreover, we test our tiny attack method on ImageNet with $\epsilon$=0.03. The adversary drops the accuracy down to $49.07\%$ (Table \ref{tab:result_v1_imagenet}), which is second only PGD and CW.

To validate the effectivity of our attack method, we repeat experiments by adjusting the $\alpha$ value of Eq. \ref{eq:attack-v1}. Fig. \ref{ex:v1_alpha} shows the effectivity of attack with different $\alpha$. It can be seen that the closer $\alpha$ is to zeros, the worse the performance of attack. Moreover, we find that when locating in [0,1], the attack achieves the best performance. As $\alpha$ gradually increases, the success of the attack first rises rapidly and then becomes stable.
\begin{figure}[ht]
	\centering
	\includegraphics[width=0.8\linewidth, height=3.1cm]{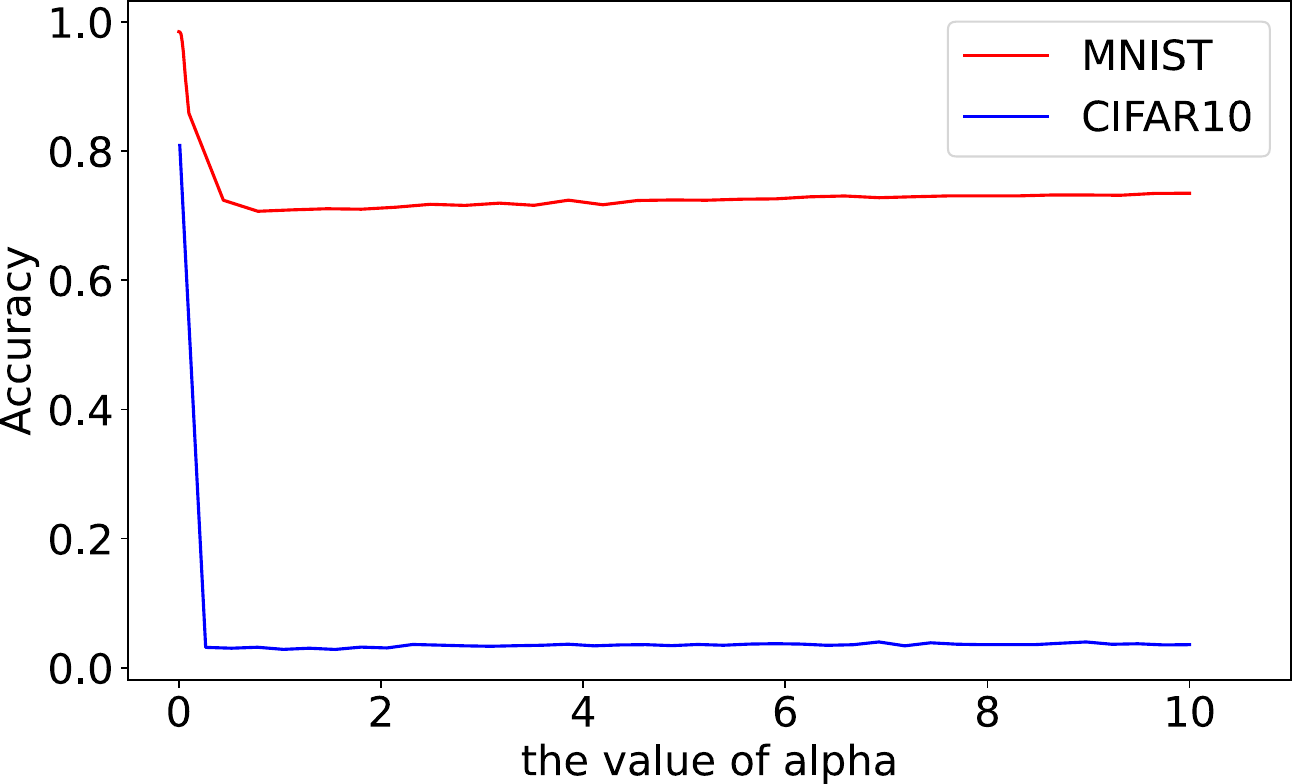}
	\caption{A trend of attack's power with various $\alpha$ values. We find that closer $\alpha$  to zero, worse performance of attack.}
	\label{ex:v1_alpha}
\end{figure}

\subsection{Unrecognizable Adversarial Examples}

In this subsection, we utilize Eq. \ref{eq:attack-v2} to create unrecognizable adversarial examples to fool VGG19 neural network, which is trained with ImageNet. We regard the images which came from ImageNet as an image from the domain (labelled as $I_{in}$). And any images which do not belong to $I_{in}$ are out of the domain (labelled as $I_{out}$). Our goal is creating some noise images and regular images which are classified by DNNs as a class with high confidence.

\paragraph{Irregular image} Firstly, we randomly generate a noise image. Then randomly choosing an image in the domain to guide the noise to close each other in \textit{Feature Space}. We find that it is easy to create such a noise image. And the DNN classifies it as a class with high confidence (Fig. \ref{overview_123}). 

\paragraph{regular image} Once again, we utilize Eq. \ref{eq:attack-v2} to generate unrecognizable adversarial examples. However, those are regular images rather than noise images. To find suitable images from out of domain, we utilize the Colored Brodatz Texture (CBT) database and the Multiband Texture (MBT) database \cite{26_abdelmounaime2013new}. CBT is a coloured version of the original 112 Brodatz grayscale texture images. And MBT is a collection of 154 colour images. The colour of the MBT images is mainly the result of inter-band and intra-band texture variation. Based on those images, we search unrecognizable adversarial examples in the out domain by minimizing the gay between the image from domain and the image from out domain in \textit{Feature Space}. Though those images are very different from that from ImageNet, the DNN still mostly believes that they are from the domain (Fig. \ref{overview_123}). 

\begin{figure}[]
	\centering
	\subfloat[]{
		\includegraphics[width=1.0\linewidth]{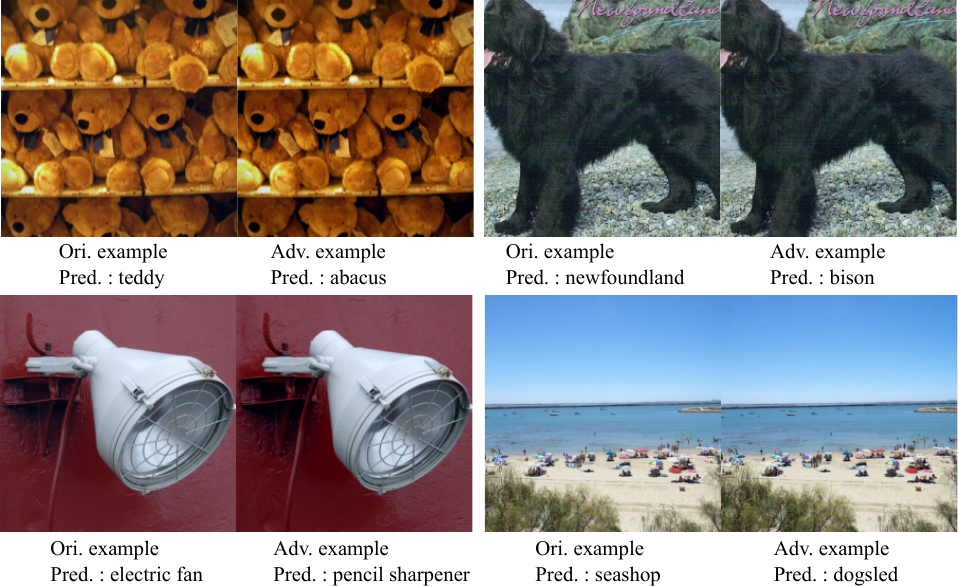}}
	
	\subfloat[]{
		\includegraphics[width=1.0\linewidth]{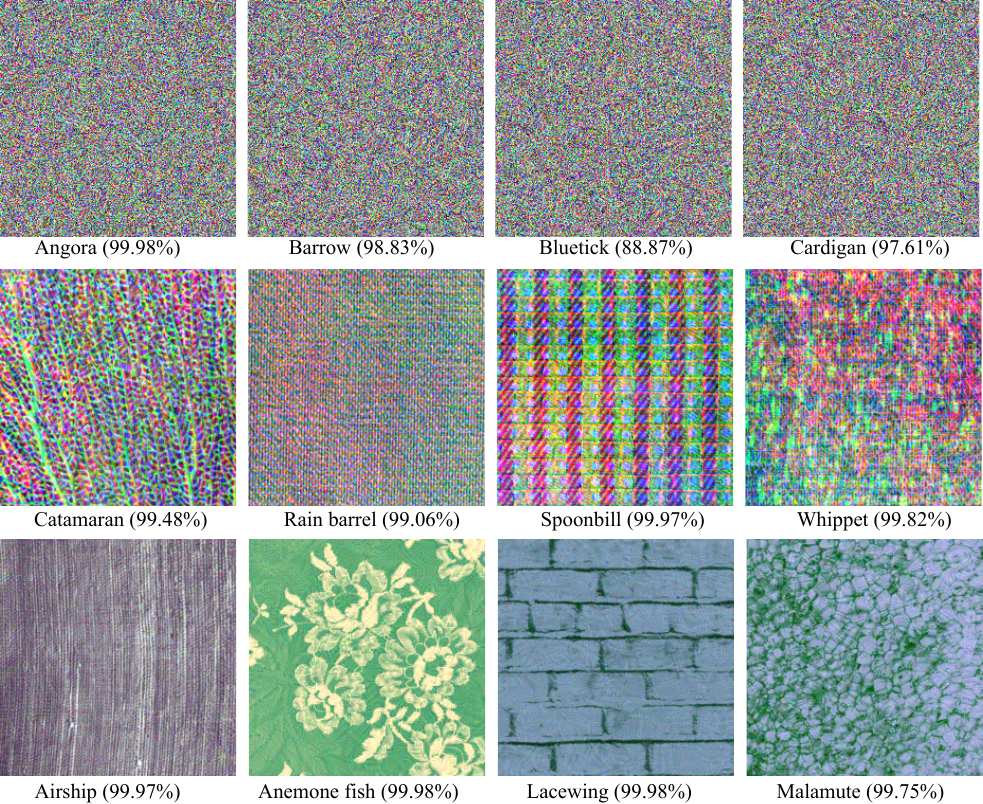}}
	
	\subfloat[]{
		\includegraphics[width=1.0\linewidth]{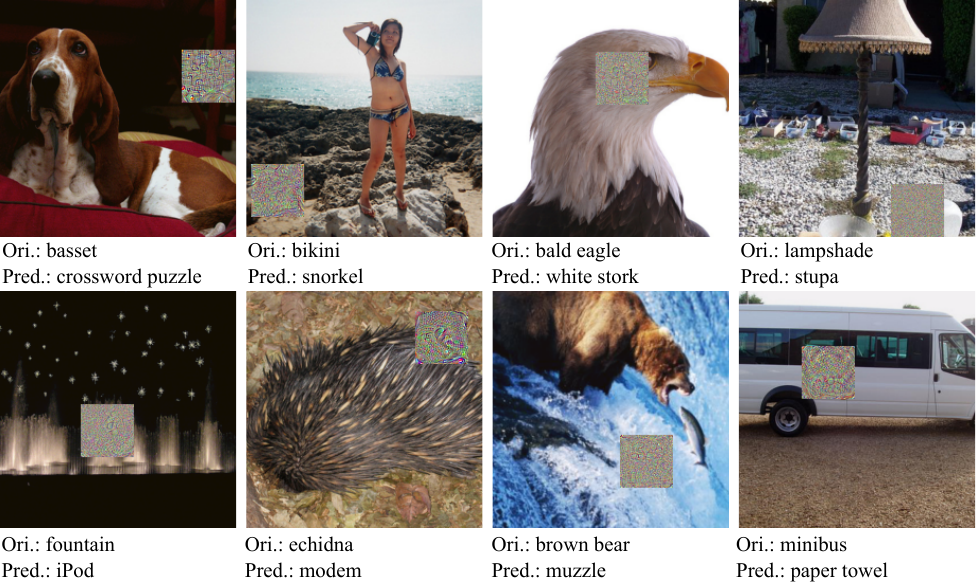}}

	\caption{There are three types of adversarial examples (e.g. imperceptible adversarial examples, unrecognizable adversarial examples and adversarial patch). (a) Some imperceptible adversarial examples are generated from ImageNet by the attack method of the first version (Eq. \ref{eq:attack-v1}). (b) Some unrecognizable adversarial examples are generated by the attack method of the second version (Eq. \ref{eq:attack-v2}). We utilize the noisy image and the image from the Colored Brodatz Texture (CBT) and the Multiband Texture (MBT) dataset \cite{26_abdelmounaime2013new} to produce the unrecognizable adversarial example which is classified a class with high confidence. (c) Some adversarial patch examples are generated by the attack method of the third version (Eq. \ref{eq:attack-v3}).}
	\label{overview_123}
\end{figure}
%
%

\subsection{Adversarial Patch}


In this subsection, we try to generate adversarial patches. Adopting the setting in \cite{16_brown2017adversarial}, we set the patch size as $q\times q$ where $q=\lfloor((Length\ of\ image)^2*c)^{0.5}\rfloor$. And $c$ is a constant. Because the image size in MNIST and CIFR10 datasets are relatively small, the generated adversarial patches are easy to modify the objective reality, so we only generate the adversarial patches to attack VGG19 network on ImageNet. The pixel of the patch can be manipulated freely. And the position of the patch is randomly located. Some examples are shown in Fig. \ref{overview_123}.

We show that our method can successfully generate adversarial patches. Since limiting facility, we randomly choose 1,000 images from validation dataset of ImageNet to generate their adversarial patches. Moreover, we test the accuracy in a different setting. The details are described in Fig. \ref{ex:v2_patch}. We find that our attack can effectively generate adversarial patches and is better than the Baseline \cite{16_brown2017adversarial} in most of the setting.
\begin{figure}[ht]
	\centering
	\includegraphics[width=0.8\linewidth, height=3.1cm]{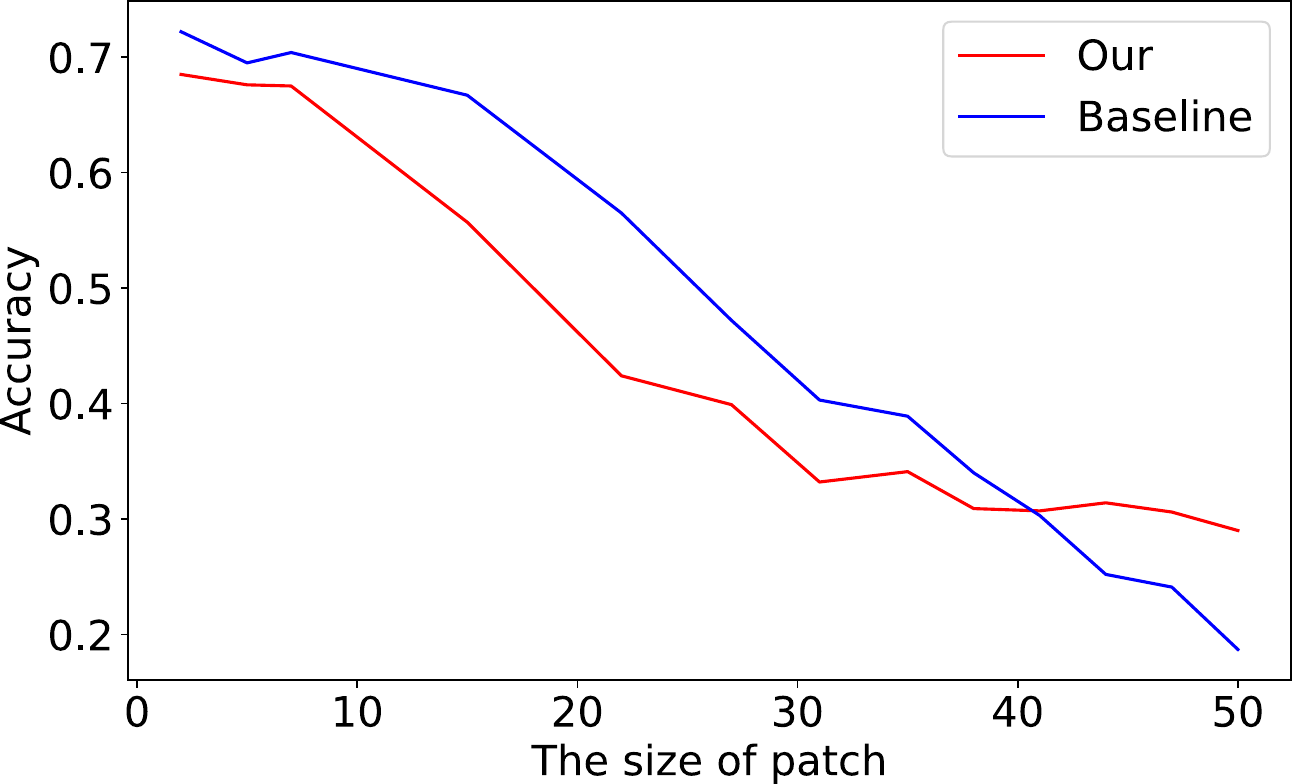}
	\caption{A trend of attack's power with various sizes of patch. We find that our attack can effectively generate adversarial patches and is better than the Baseline \cite{16_brown2017adversarial} in most of setting.}
	\label{ex:v2_patch}
\end{figure}


\section{Conclusion} \label{conclusion}

We show that, it is significantly different between human recognizes an image and DNNS to recognize the image. The influence on the prediction results depends on whether the original characteristic semantics are changed in the hidden layer under the condition that adversarial perturbation does not change objective reality. Therefore, based on this conjecture, incorporating hidden layer representation we propose a defence strategy and three types of adversarial attacks.

Our defence strategy achieves state-of-the-art performance and is less time-consuming than adversarial training. However, even if our defence method can defend against the attack, the improper setting will lead to the model’s accuracy decreased in the clean sample. Therefore, there is room for improvement in our defence in the future.
\section*{Acknowledgment}
This work was supported in part by Key Project of Natural Science Foundation of China (Grant 61732011), in part by the National Natural Science Foundation of China (Grants 61976141, 61772344 and 61732011), in part by the Natural Science Foundation of SZU (827-000230), and in part by the Interdisciplinary Innovation Team of Shenzhen University.


\ifCLASSOPTIONcaptionsoff
  \newpage
\fi



%
\bibliographystyle{IEEEtran}
\bibliography{reference}

\begin{thebibliography}{10}
\providecommand{\url}[1]{#1}
\csname url@samestyle\endcsname
\providecommand{\newblock}{\relax}
\providecommand{\bibinfo}[2]{#2}
\providecommand{\BIBentrySTDinterwordspacing}{\spaceskip=0pt\relax}
\providecommand{\BIBentryALTinterwordstretchfactor}{4}
\providecommand{\BIBentryALTinterwordspacing}{\spaceskip=\fontdimen2\font plus
\BIBentryALTinterwordstretchfactor\fontdimen3\font minus
  \fontdimen4\font\relax}
\providecommand{\BIBforeignlanguage}[2]{{%
\expandafter\ifx\csname l@#1\endcsname\relax
\typeout{** WARNING: IEEEtran.bst: No hyphenation pattern has been}%
\typeout{** loaded for the language `#1'. Using the pattern for}%
\typeout{** the default language instead.}%
\else
\language=\csname l@#1\endcsname
\fi
#2}}
\providecommand{\BIBdecl}{\relax}
\BIBdecl

\bibitem{45_lecun1998gradient}
Y.~LeCun, L.~Bottou, Y.~Bengio, and P.~Haffner, ``Gradient-based learning
  applied to document recognition,'' \emph{Proceedings of the IEEE}, vol.~86,
  no.~11, pp. 2278--2324, 1998.

\bibitem{33_krizhevsky2012imagenet}
A.~Krizhevsky, I.~Sutskever, and G.~E. Hinton, ``Imagenet classification with
  deep convolutional neural networks,'' in \emph{Advances in neural information
  processing systems}, 2012, pp. 1097--1105.

\bibitem{42_he2016deep}
K.~He, X.~Zhang, S.~Ren, and J.~Sun, ``Deep residual learning for image
  recognition,'' in \emph{Proceedings of the IEEE conference on computer vision
  and pattern recognition}, 2016, pp. 770--778.

\bibitem{46_olah2015understanding}
C.~Olah, ``Understanding lstm networks,'' 2015.

\bibitem{2_szegedy2013intriguing}
\BIBentryALTinterwordspacing
C.~Szegedy, W.~Zaremba, I.~Sutskever, J.~Bruna, D.~Erhan, I.~J. Goodfellow, and
  R.~Fergus, ``Intriguing properties of neural networks,'' in \emph{2nd
  International Conference on Learning Representations, {ICLR} 2014, Banff, AB,
  Canada, April 14-16, 2014, Conference Track Proceedings}, Y.~Bengio and
  Y.~LeCun, Eds., 2014. [Online]. Available:
  \url{http://arxiv.org/abs/1312.6199}
\BIBentrySTDinterwordspacing

\bibitem{6_goodfellow2014explaining}
\BIBentryALTinterwordspacing
I.~J. Goodfellow, J.~Shlens, and C.~Szegedy, ``Explaining and harnessing
  adversarial examples,'' in \emph{3rd International Conference on Learning
  Representations, {ICLR} 2015, San Diego, CA, USA, May 7-9, 2015, Conference
  Track Proceedings}, Y.~Bengio and Y.~LeCun, Eds., 2015. [Online]. Available:
  \url{http://arxiv.org/abs/1412.6572}
\BIBentrySTDinterwordspacing

\bibitem{7_kurakin2016adversarial}
\BIBentryALTinterwordspacing
A.~Kurakin, I.~J. Goodfellow, and S.~Bengio, ``Adversarial examples in the
  physical world,'' in \emph{5th International Conference on Learning
  Representations, {ICLR} 2017, Toulon, France, April 24-26, 2017, Workshop
  Track Proceedings}.\hskip 1em plus 0.5em minus 0.4em\relax OpenReview.net,
  2017. [Online]. Available: \url{https://openreview.net/forum?id=HJGU3Rodl}
\BIBentrySTDinterwordspacing

\bibitem{8_papernot2016limitations}
N.~Papernot, P.~McDaniel, S.~Jha, M.~Fredrikson, Z.~B. Celik, and A.~Swami,
  ``The limitations of deep learning in adversarial settings,'' in \emph{2016
  IEEE European symposium on security and privacy (EuroS\&P)}.\hskip 1em plus
  0.5em minus 0.4em\relax IEEE, 2016, pp. 372--387.

\bibitem{9_moosavi2016deepfool}
S.-M. Moosavi-Dezfooli, A.~Fawzi, and P.~Frossard, ``Deepfool: a simple and
  accurate method to fool deep neural networks,'' in \emph{Proceedings of the
  IEEE conference on computer vision and pattern recognition}, 2016, pp.
  2574--2582.

\bibitem{10_carlini2017towards}
N.~Carlini and D.~Wagner, ``Towards evaluating the robustness of neural
  networks,'' in \emph{2017 ieee symposium on security and privacy (sp)}.\hskip
  1em plus 0.5em minus 0.4em\relax IEEE, 2017, pp. 39--57.

\bibitem{11_chen2017zoo}
P.-Y. Chen, H.~Zhang, Y.~Sharma, J.~Yi, and C.-J. Hsieh, ``Zoo: Zeroth order
  optimization based black-box attacks to deep neural networks without training
  substitute models,'' in \emph{Proceedings of the 10th ACM Workshop on
  Artificial Intelligence and Security}, 2017, pp. 15--26.

\bibitem{12_moosavi2017universal}
S.-M. Moosavi-Dezfooli, A.~Fawzi, O.~Fawzi, and P.~Frossard, ``Universal
  adversarial perturbations,'' in \emph{Proceedings of the IEEE conference on
  computer vision and pattern recognition}, 2017, pp. 1765--1773.

\bibitem{13_su2019one}
J.~Su, D.~V. Vargas, and K.~Sakurai, ``One pixel attack for fooling deep neural
  networks,'' \emph{IEEE Transactions on Evolutionary Computation}, vol.~23,
  no.~5, pp. 828--841, 2019.

\bibitem{14_zhao2017generating}
Z.~Zhao, D.~Dua, and S.~Singh, ``Generating natural adversarial examples,''
  \emph{arXiv preprint arXiv:1710.11342}, 2017.

\bibitem{15_tabacof2016adversarial}
P.~Tabacof, J.~Tavares, and E.~Valle, ``Adversarial images for variational
  autoencoders,'' \emph{arXiv preprint arXiv:1612.00155}, 2016.

\bibitem{24_madry2017towards}
\BIBentryALTinterwordspacing
A.~Madry, A.~Makelov, L.~Schmidt, D.~Tsipras, and A.~Vladu, ``Towards deep
  learning models resistant to adversarial attacks,'' in \emph{6th
  International Conference on Learning Representations, {ICLR} 2018, Vancouver,
  BC, Canada, April 30 - May 3, 2018, Conference Track Proceedings}.\hskip 1em
  plus 0.5em minus 0.4em\relax OpenReview.net, 2018. [Online]. Available:
  \url{https://openreview.net/forum?id=rJzIBfZAb}
\BIBentrySTDinterwordspacing

\bibitem{43_tramer2017ensemble}
\BIBentryALTinterwordspacing
F.~Tram{\`{e}}r, A.~Kurakin, N.~Papernot, I.~J. Goodfellow, D.~Boneh, and P.~D.
  McDaniel, ``Ensemble adversarial training: Attacks and defenses,'' in
  \emph{6th International Conference on Learning Representations, {ICLR} 2018,
  Vancouver, BC, Canada, April 30 - May 3, 2018, Conference Track
  Proceedings}.\hskip 1em plus 0.5em minus 0.4em\relax OpenReview.net, 2018.
  [Online]. Available: \url{https://openreview.net/forum?id=rkZvSe-RZ}
\BIBentrySTDinterwordspacing

\bibitem{44_papernot2016distillation}
N.~Papernot, P.~McDaniel, X.~Wu, S.~Jha, and A.~Swami, ``Distillation as a
  defense to adversarial perturbations against deep neural networks,'' in
  \emph{2016 IEEE Symposium on Security and Privacy (SP)}.\hskip 1em plus 0.5em
  minus 0.4em\relax IEEE, 2016, pp. 582--597.

\bibitem{50_huang2015learning}
R.~Huang, B.~Xu, D.~Schuurmans, and C.~Szepesv{\'a}ri, ``Learning with a strong
  adversary,'' \emph{arXiv preprint arXiv:1511.03034}, 2015.

\bibitem{51_cheng2020cat}
M.~Cheng, Q.~Lei, P.-Y. Chen, I.~Dhillon, and C.-J. Hsieh, ``Cat: Customized
  adversarial training for improved robustness,'' \emph{arXiv preprint
  arXiv:2002.06789}, 2020.

\bibitem{52_song2018improving}
\BIBentryALTinterwordspacing
C.~Song, K.~He, L.~Wang, and J.~E. Hopcroft, ``Improving the generalization of
  adversarial training with domain adaptation,'' in \emph{7th International
  Conference on Learning Representations, {ICLR} 2019, New Orleans, LA, USA,
  May 6-9, 2019}.\hskip 1em plus 0.5em minus 0.4em\relax OpenReview.net, 2019.
  [Online]. Available: \url{https://openreview.net/forum?id=SyfIfnC5Ym}
\BIBentrySTDinterwordspacing

\bibitem{53_zhang2019theoretically}
H.~Zhang, Y.~Yu, J.~Jiao, E.~P. Xing, L.~E. Ghaoui, and M.~I. Jordan,
  ``Theoretically principled trade-off between robustness and accuracy,''
  \emph{arXiv preprint arXiv:1901.08573}, 2019.

\bibitem{16_brown2017adversarial}
\BIBentryALTinterwordspacing
T.~B. Brown, D.~Man{\'{e}}, A.~Roy, M.~Abadi, and J.~Gilmer, ``Adversarial
  patch,'' \emph{CoRR}, vol. abs/1712.09665, 2017. [Online]. Available:
  \url{http://arxiv.org/abs/1712.09665}
\BIBentrySTDinterwordspacing

\bibitem{17_liu2018dpatch}
\BIBentryALTinterwordspacing
X.~Liu, H.~Yang, Z.~Liu, L.~Song, Y.~Chen, and H.~Li, ``{DPATCH:} an
  adversarial patch attack on object detectors,'' in \emph{Workshop on
  Artificial Intelligence Safety 2019 co-located with the Thirty-Third {AAAI}
  Conference on Artificial Intelligence 2019 (AAAI-19), Honolulu, Hawaii,
  January 27, 2019}, ser. {CEUR} Workshop Proceedings, H.~Espinoza, S.~{\'{O}}.
  h{\'{E}}igeartaigh, X.~Huang, J.~Hern{\'{a}}ndez{-}Orallo, and
  M.~Castillo{-}Effen, Eds., vol. 2301.\hskip 1em plus 0.5em minus 0.4em\relax
  CEUR-WS.org, 2019. [Online]. Available:
  \url{http://ceur-ws.org/Vol-2301/paper\_5.pdf}
\BIBentrySTDinterwordspacing

\bibitem{18_karmon2018lavan}
\BIBentryALTinterwordspacing
D.~Karmon, D.~Zoran, and Y.~Goldberg, ``Lavan: Localized and visible
  adversarial noise,'' in \emph{Proceedings of the 35th International
  Conference on Machine Learning, {ICML} 2018, Stockholmsm{\"{a}}ssan,
  Stockholm, Sweden, July 10-15, 2018}, ser. Proceedings of Machine Learning
  Research, J.~G. Dy and A.~Krause, Eds., vol.~80.\hskip 1em plus 0.5em minus
  0.4em\relax {PMLR}, 2018, pp. 2512--2520. [Online]. Available:
  \url{http://proceedings.mlr.press/v80/karmon18a.html}
\BIBentrySTDinterwordspacing

\bibitem{19_lee2019physical}
M.~Lee and Z.~Kolter, ``On physical adversarial patches for object detection,''
  \emph{arXiv preprint arXiv:1906.11897}, 2019.

\bibitem{4_nguyen2015deep}
A.~Nguyen, J.~Yosinski, and J.~Clune, ``Deep neural networks are easily fooled:
  High confidence predictions for unrecognizable images,'' in \emph{Proceedings
  of the IEEE conference on computer vision and pattern recognition}, 2015, pp.
  427--436.

\bibitem{55_DBLP:journals/corr/SabourCFF15}
\BIBentryALTinterwordspacing
S.~Sabour, Y.~Cao, F.~Faghri, and D.~J. Fleet, ``Adversarial manipulation of
  deep representations,'' in \emph{4th International Conference on Learning
  Representations, {ICLR} 2016, San Juan, Puerto Rico, May 2-4, 2016,
  Conference Track Proceedings}, Y.~Bengio and Y.~LeCun, Eds., 2016. [Online].
  Available: \url{http://arxiv.org/abs/1511.05122}
\BIBentrySTDinterwordspacing

\bibitem{34_ren2015faster}
S.~Ren, K.~He, R.~Girshick, and J.~Sun, ``Faster r-cnn: Towards real-time
  object detection with region proposal networks,'' in \emph{Advances in neural
  information processing systems}, 2015, pp. 91--99.

\bibitem{31_sutskever2014sequence}
I.~Sutskever, O.~Vinyals, and Q.~V. Le, ``Sequence to sequence learning with
  neural networks,'' in \emph{Advances in neural information processing
  systems}, 2014, pp. 3104--3112.

\bibitem{32_xu2016text}
H.~Xu, M.~Dong, D.~Zhu, A.~Kotov, A.~I. Carcone, and S.~Naar-King, ``Text
  classification with topic-based word embedding and convolutional neural
  networks,'' in \emph{Proceedings of the 7th ACM International Conference on
  Bioinformatics, Computational Biology, and Health Informatics}, 2016, pp.
  88--97.

\bibitem{37_xie2017adversarial}
C.~Xie, J.~Wang, Z.~Zhang, Y.~Zhou, L.~Xie, and A.~Yuille, ``Adversarial
  examples for semantic segmentation and object detection,'' in
  \emph{Proceedings of the IEEE International Conference on Computer Vision},
  2017, pp. 1369--1378.

\bibitem{22_thys2019fooling}
S.~Thys, W.~Van~Ranst, and T.~Goedem{\'e}, ``Fooling automated surveillance
  cameras: adversarial patches to attack person detection,'' in
  \emph{Proceedings of the IEEE Conference on Computer Vision and Pattern
  Recognition Workshops}, 2019, pp. 0--0.

\bibitem{36_hendrik2017universal}
\BIBentryALTinterwordspacing
V.~Fischer, M.~C. Kumar, J.~H. Metzen, and T.~Brox, ``Adversarial examples for
  semantic image segmentation,'' in \emph{5th International Conference on
  Learning Representations, {ICLR} 2017, Toulon, France, April 24-26, 2017,
  Workshop Track Proceedings}.\hskip 1em plus 0.5em minus 0.4em\relax
  OpenReview.net, 2017. [Online]. Available:
  \url{https://openreview.net/forum?id=S1SED1MYe}
\BIBentrySTDinterwordspacing

\bibitem{38_fischer2017adversarial}
------, ``Adversarial examples for semantic image segmentation,'' \emph{arXiv
  preprint arXiv:1703.01101}, 2017.

\bibitem{27_mikolov2011extensions}
T.~Mikolov, S.~Kombrink, L.~Burget, J.~{\v{C}}ernock{\`y}, and S.~Khudanpur,
  ``Extensions of recurrent neural network language model,'' in \emph{2011 IEEE
  international conference on acoustics, speech and signal processing
  (ICASSP)}.\hskip 1em plus 0.5em minus 0.4em\relax IEEE, 2011, pp. 5528--5531.

\bibitem{28_jozefowicz2016exploring}
R.~Jozefowicz, O.~Vinyals, M.~Schuster, N.~Shazeer, and Y.~Wu, ``Exploring the
  limits of language modeling,'' \emph{arXiv preprint arXiv:1602.02410}, 2016.

\bibitem{29_kiperwasser2016simple}
E.~Kiperwasser and Y.~Goldberg, ``Simple and accurate dependency parsing using
  bidirectional lstm feature representations,'' \emph{Transactions of the
  Association for Computational Linguistics}, vol.~4, pp. 313--327, 2016.

\bibitem{30_bahdanau2014neural}
\BIBentryALTinterwordspacing
D.~Bahdanau, K.~Cho, and Y.~Bengio, ``Neural machine translation by jointly
  learning to align and translate,'' in \emph{3rd International Conference on
  Learning Representations, {ICLR} 2015, San Diego, CA, USA, May 7-9, 2015,
  Conference Track Proceedings}, Y.~Bengio and Y.~LeCun, Eds., 2015. [Online].
  Available: \url{http://arxiv.org/abs/1409.0473}
\BIBentrySTDinterwordspacing

\bibitem{35_papernot2016crafting}
N.~Papernot, P.~McDaniel, A.~Swami, and R.~Harang, ``Crafting adversarial input
  sequences for recurrent neural networks,'' in \emph{MILCOM 2016-2016 IEEE
  Military Communications Conference}.\hskip 1em plus 0.5em minus 0.4em\relax
  IEEE, 2016, pp. 49--54.

\bibitem{39_russakovsky2015imagenet}
O.~Russakovsky, J.~Deng, H.~Su, J.~Krause, S.~Satheesh, S.~Ma, Z.~Huang,
  A.~Karpathy, A.~Khosla, M.~Bernstein \emph{et~al.}, ``Imagenet large scale
  visual recognition challenge,'' \emph{International journal of computer
  vision}, vol. 115, no.~3, pp. 211--252, 2015.

\bibitem{40_szegedy2015going}
C.~Szegedy, W.~Liu, Y.~Jia, P.~Sermanet, S.~Reed, D.~Anguelov, D.~Erhan,
  V.~Vanhoucke, and A.~Rabinovich, ``Going deeper with convolutions,'' in
  \emph{Proceedings of the IEEE conference on computer vision and pattern
  recognition}, 2015, pp. 1--9.

\bibitem{41_simonyan2014very}
\BIBentryALTinterwordspacing
K.~Simonyan and A.~Zisserman, ``Very deep convolutional networks for
  large-scale image recognition,'' in \emph{3rd International Conference on
  Learning Representations, {ICLR} 2015, San Diego, CA, USA, May 7-9, 2015,
  Conference Track Proceedings}, Y.~Bengio and Y.~LeCun, Eds., 2015. [Online].
  Available: \url{http://arxiv.org/abs/1409.1556}
\BIBentrySTDinterwordspacing

\bibitem{20_sharif2016accessorize}
M.~Sharif, S.~Bhagavatula, L.~Bauer, and M.~K. Reiter, ``Accessorize to a
  crime: Real and stealthy attacks on state-of-the-art face recognition,'' in
  \emph{Proceedings of the 2016 acm sigsac conference on computer and
  communications security}, 2016, pp. 1528--1540.

\bibitem{21_eykholt2018robust}
K.~Eykholt, I.~Evtimov, E.~Fernandes, B.~Li, A.~Rahmati, C.~Xiao, A.~Prakash,
  T.~Kohno, and D.~Song, ``Robust physical-world attacks on deep learning
  visual classification,'' in \emph{Proceedings of the IEEE Conference on
  Computer Vision and Pattern Recognition}, 2018, pp. 1625--1634.

\bibitem{3_liu2019perceptual}
A.~Liu, X.~Liu, J.~Fan, Y.~Ma, A.~Zhang, H.~Xie, and D.~Tao,
  ``Perceptual-sensitive gan for generating adversarial patches,'' in
  \emph{Proceedings of the AAAI Conference on Artificial Intelligence},
  vol.~33, 2019, pp. 1028--1035.

\bibitem{23_chen2018shapeshifter}
S.-T. Chen, C.~Cornelius, J.~Martin, and D.~H.~P. Chau, ``Shapeshifter: Robust
  physical adversarial attack on faster r-cnn object detector,'' in \emph{Joint
  European Conference on Machine Learning and Knowledge Discovery in
  Databases}.\hskip 1em plus 0.5em minus 0.4em\relax Springer, 2018, pp.
  52--68.

\bibitem{25_cully2014robots}
A.~Cully, J.~Clune, and J.-B. Mouret, ``Robots that can adapt like natural
  animals,'' \emph{arXiv preprint arXiv:1407.3501}, vol.~2, 2014.

\bibitem{1_sabour2015adversarial}
\BIBentryALTinterwordspacing
S.~Sabour, Y.~Cao, F.~Faghri, and D.~J. Fleet, ``Adversarial manipulation of
  deep representations,'' in \emph{4th International Conference on Learning
  Representations, {ICLR} 2016, San Juan, Puerto Rico, May 2-4, 2016,
  Conference Track Proceedings}, Y.~Bengio and Y.~LeCun, Eds., 2016. [Online].
  Available: \url{http://arxiv.org/abs/1511.05122}
\BIBentrySTDinterwordspacing

\bibitem{54_athalye2018obfuscated}
\BIBentryALTinterwordspacing
A.~Athalye, N.~Carlini, and D.~A. Wagner, ``Obfuscated gradients give a false
  sense of security: Circumventing defenses to adversarial examples,'' in
  \emph{Proceedings of the 35th International Conference on Machine Learning,
  {ICML} 2018, Stockholmsm{\"{a}}ssan, Stockholm, Sweden, July 10-15, 2018},
  ser. Proceedings of Machine Learning Research, J.~G. Dy and A.~Krause, Eds.,
  vol.~80.\hskip 1em plus 0.5em minus 0.4em\relax {PMLR}, 2018, pp. 274--283.
  [Online]. Available: \url{http://proceedings.mlr.press/v80/athalye18a.html}
\BIBentrySTDinterwordspacing

\bibitem{49_ding2019advertorch}
G.~W. Ding, L.~Wang, and X.~Jin, ``{AdverTorch} v0.1: An adversarial robustness
  toolbox based on pytorch,'' \emph{arXiv preprint arXiv:1902.07623}, 2019.

\bibitem{47_deng2012mnist}
L.~Deng, ``The mnist database of handwritten digit images for machine learning
  research [best of the web],'' \emph{IEEE Signal Processing Magazine},
  vol.~29, no.~6, pp. 141--142, 2012.

\bibitem{48_krizhevsky2009learning}
A.~Krizhevsky, G.~Hinton \emph{et~al.}, ``Learning multiple layers of features
  from tiny images,'' 2009.

\bibitem{26_abdelmounaime2013new}
S.~Abdelmounaime and H.~Dong-Chen, ``New brodatz-based image databases for
  grayscale color and multiband texture analysis,'' \emph{ISRN Machine Vision},
  vol. 2013, 2013.

\end{thebibliography}
\end{document}